\documentclass[letterpaper]{article}
\pdfoutput=1
\usepackage{uai2019}
\usepackage[margin=1in]{geometry}

\usepackage{times}

\usepackage{microtype}
\usepackage{booktabs} 
\usepackage{epsf}
\usepackage{epsfig}
\usepackage{subfigure}
\usepackage{amsmath}
\usepackage{amssymb}
\usepackage{amsfonts}
\usepackage{amsthm}
\usepackage{mathtools}
\usepackage{multirow}
\usepackage{multicol}
\usepackage{graphicx} 
\usepackage{xspace}
\usepackage{hyperref}
\usepackage{listings}
\usepackage{xcolor}
\usepackage{enumitem}
\usepackage{hyperref}
\usepackage{algorithm}
\usepackage{algorithmic}
\usepackage{natbib}
\setcitestyle{authoryear,round,citesep={;},aysep={,},yysep={;}}

\usepackage{verbatim}
\usepackage{comment}

\title{Fall of Empires: Breaking Byzantine-tolerant SGD by Inner Product Manipulation}

\author{} 

\author{ {\bf Cong Xie} \\
Computer Science Dept. \\
UIUC \\
\And
{\bf Sanmi Koyejo}  \\
Computer Science Dept. \\
UIUC \\
\And
{\bf Indranil Gupta}   \\
Computer Science Dept. \\
UIUC \\
}

\begin{document}

\def\Blue{\color{blue}}
\def\Purple{\color{purple}}

\def\A{{\bf A}}
\def\a{{\bf a}}
\def\B{{\bf B}}
\def\C{{\bf C}}
\def\c{{\bf c}}
\def\D{{\bf D}}
\def\d{{\bf d}}
\def\E{{\mathbb{E}}}
\def\F{{\bf F}}
\def\e{{\bf e}}
\def\f{{\bf f}}
\def\G{{\bf G}}
\def\H{{\bf H}}
\def\I{{\bf I}}
\def\K{{\bf K}}
\def\L{{\bf L}}
\def\M{{\bf M}}
\def\m{{\bf m}}
\def\N{{\bf N}}
\def\n{{\bf n}}
\def\Q{{\bf Q}}
\def\q{{\bf q}}
\def\R{{\mathbb{R}}}
\def\S{{\bf S}}
\def\s{{\bf s}}
\def\T{{\bf T}}
\def\U{{\bf U}}
\def\u{{\bf u}}
\def\V{{\bf V}}
\def\v{{\bf v}}
\def\W{{\bf W}}
\def\w{{\bf w}}
\def\X{{\bf X}}
\def\x{{\bf x}}
\def\bx{{\bar{x}}}
\def\Y{{\bf Y}}
\def\y{{\bf y}}
\def\Z{{\bf Z}}
\def\z{{\bf z}}
\def\0{{\bf 0}}
\def\1{{\bf 1}}

\def\AM{{\mathcal A}}
\def\CM{{\mathcal C}}
\def\DM{{\mathcal D}}
\def\GM{{\mathcal G}}
\def\FM{{\mathcal F}}
\def\IM{{\mathcal I}}
\def\NM{{\mathcal N}}
\def\OM{{\mathcal O}}
\def\SM{{\mathcal S}}
\def\TM{{\mathcal T}}
\def\UM{{\mathcal U}}
\def\XM{{\mathcal X}}
\def\YM{{\mathcal Y}}
\def\RB{{\mathbb R}}

\def\TX{\tilde{\bf X}}
\def\tx{\tilde{\bf x}}
\def\ty{\tilde{\bf y}}
\def\TZ{\tilde{\bf Z}}
\def\tz{\tilde{\bf z}}
\def\hd{\hat{d}}
\def\HD{\hat{\bf D}}
\def\hx{\hat{\bf x}}
\def\TD{\tilde{\Delta}}
\def\tg{\tilde{g}}
\def\tmu{\tilde{\mu}}

\def\alp{\mbox{\boldmath$\alpha$\unboldmath}}
\def\bet{\mbox{\boldmath$\beta$\unboldmath}}
\def\epsi{\mbox{\boldmath$\epsilon$\unboldmath}}
\def\etab{\mbox{\boldmath$\eta$\unboldmath}}
\def\ph{\mbox{\boldmath$\phi$\unboldmath}}
\def\pii{\mbox{\boldmath$\pi$\unboldmath}}
\def\Ph{\mbox{\boldmath$\Phi$\unboldmath}}
\def\Ps{\mbox{\boldmath$\Psi$\unboldmath}}
\def\tha{\mbox{\boldmath$\theta$\unboldmath}}
\def\Tha{\mbox{\boldmath$\Theta$\unboldmath}}
\def\muu{\mbox{\boldmath$\mu$\unboldmath}}
\def\Si{\mbox{\boldmath$\Sigma$\unboldmath}}
\def\si{\mbox{\boldmath$\sigma$\unboldmath}}
\def\Gam{\mbox{\boldmath$\Gamma$\unboldmath}}
\def\Lam{\mbox{\boldmath$\Lambda$\unboldmath}}
\def\De{\mbox{\boldmath$\Delta$\unboldmath}}
\def\Ome{\mbox{\boldmath$\Omega$\unboldmath}}
\def\TOme{\mbox{\boldmath$\hat{\Omega}$\unboldmath}}
\def\vps{\mbox{\boldmath$\varepsilon$\unboldmath}}
\newcommand{\ti}[1]{\tilde{#1}}
\def\Ncal{\mathcal{N}}
\def\argmax{\mathop{\rm argmax}}
\def\argmin{\mathop{\rm argmin}}
\providecommand{\abs}[1]{\lvert#1\rvert}
\providecommand{\norm}[2]{\lVert#1\rVert_{#2}}

\def\Zs{{\Z_{\mathrm{S}}}}
\def\Zl{{\Z_{\mathrm{L}}}}
\def\Yr{{\Y_{\mathrm{R}}}}
\def\Yg{{\Y_{\mathrm{G}}}}
\def\Yb{{\Y_{\mathrm{B}}}}
\def\Ar{{\A_{\mathrm{R}}}}
\def\Ag{{\A_{\mathrm{G}}}}
\def\Ab{{\A_{\mathrm{B}}}}
\def\As{{\A_{\mathrm{S}}}}
\def\Asr{{\A_{\mathrm{S}_{\mathrm{R}}}}}
\def\Asg{{\A_{\mathrm{S}_{\mathrm{G}}}}}
\def\Asb{{\A_{\mathrm{S}_{\mathrm{B}}}}}
\def\Or{{\Ome_{\mathrm{R}}}}
\def\Og{{\Ome_{\mathrm{G}}}}
\def\Ob{{\Ome_{\mathrm{B}}}}

\def\Expect{\mathbb{E}}

\def\vec{\mathrm{vec}}
\def\fold{\mathrm{fold}}
\def\index{\mathrm{index}}
\def\sgn{\mathrm{sgn}}
\def\tr{\mathrm{tr}}
\def\rk{\mathrm{rank}}
\def\diag{\mathsf{diag}}
\def\const{\mathrm{Const}}
\def\dg{\mathsf{dg}}
\def\st{\mathsf{s.t.}}
\def\vect{\mathsf{vec}}
\def\MCAR{\mathrm{MCAR}}
\def\MSAR{\mathrm{MSAR}}
\def\etal{{\em et al.\/}\,}
\def\prox{\mathrm{prox}^h_\gamma}
\newcommand{\indep}{{\;\bot\!\!\!\!\!\!\bot\;}}

\newcommand{\mtrxt}[1]{{#1}^\top}
\newcommand{\mtrx}[4]{\left[\begin{matrix}#1 & #2 \\ #3 & #4\end{matrix}\right]}
\DeclarePairedDelimiter\vnorm{\lVert}{\rVert}
\DeclarePairedDelimiterX{\innerprod}[2]{\langle}{\rangle}{#1, #2}

\def\Lsize{\hbox{\space \raise-2mm\hbox{$\textstyle \L \atop \scriptstyle {m\times 3n}$} \space}}
\def\Ssize{\hbox{\space \raise-2mm\hbox{$\textstyle \S \atop \scriptstyle {m\times 3n}$} \space}}
\def\Osize{\hbox{\space \raise-2mm\hbox{$\textstyle \Ome \atop \scriptstyle {m\times 3n}$} \space}}
\def\Tsize{\hbox{\space \raise-2mm\hbox{$\textstyle \T \atop \scriptstyle {3n\times n}$} \space}}
\def\Bsize{\hbox{\space \raise-2mm\hbox{$\textstyle \B \atop \scriptstyle {m\times n}$} \space}}

\newcommand{\twopartdef}[4]
{
	\left\{
		\begin{array}{ll}
			#1 & \mbox{if } #2 \\
			#3 & \mbox{if } #4
		\end{array}
	\right.
}

\newcommand{\tabincell}[2]{\begin{tabular}{@{}#1@{}}#2\end{tabular}}

\newcommand{\NewProcedure}[1]{\STATE\hspace{-\algorithmicindent} {\large\underline{\textbf{{#1}:}}}
\setcounter{ALC@line}{0}}
\newcommand{\NewThread}[1]{\STATE\hspace{-\algorithmicindent} {\quad\large\underline{\textbf{{#1}:}}} 
\setcounter{ALC@line}{0}}

\DeclarePairedDelimiter\ceil{\lceil}{\rceil}
\DeclarePairedDelimiter\floor{\lfloor}{\rfloor}

\newcommand{\ip}[2]{\left\langle #1, #2 \right \rangle}

\newtheorem{theorem}{Theorem}
\newtheorem{lemma}{Lemma}
\newtheorem{corollary}{Corollary}
\newtheorem{assumption}{Assumption}
\newtheorem{definition}{Definition}
\newtheorem{proposition}{Proposition}
\newtheorem{remark}{Remark}

\newcommand{\pfcomment}[1]{
\tag*{$\triangleright$ #1}
}

\def\aggr{{\tt Aggr}}
\def\mean{{\tt Mean}}
\def\median{{\tt Median}}
\def\krum{{\tt Krum}}
\def\zeno{{\tt Zeno}}

\maketitle

\begin{abstract}
Recently, new defense techniques have been developed to tolerate Byzantine failures for distributed machine learning. The Byzantine model captures  workers that behave arbitrarily, including malicious and compromised workers. In this paper, we  break two prevailing Byzantine-tolerant techniques. Specifically we show robust aggregation methods for synchronous SGD -- coordinate-wise median and Krum -- can be broken using new attack strategies based on inner product manipulation. We prove our results theoretically, as well as show empirical validation.  
\end{abstract}

\section{INTRODUCTION}

The security of distributed machine learning has drawn increasing attention 
in recent years. Among the threat models, Byzantine failures~\citep{Lamport1982TheBG} are perhaps the most well-studied. In the Byzantine model, workers can behave arbitrarily and maliciously. In addition, Byzantine workers are omniscient and can conspire.
Most of the existing Byzantine-tolerant machine-learning algorithms~\citep{blanchard2017machine,Chen2017DistributedSM,yin2018byzantine,feng2014distributed,Su2016FaultTolerantMO,su2016defending,alistarh2018byzantine} focus on the protection of distributed Stochastic Gradient Descent~(SGD).

In this paper, we consider Byzantine-tolerant SGD in a server-worker architecture~(also known as the parameter server architecture~\citep{li2014scaling,li2014communication}), depicted in Figure~\ref{fig:ps}. The  system is composed of  server nodes and worker nodes. In each epoch, the workers pull the latest model from the servers, estimate the gradients using the locally sampled training data, and then push the gradient estimators to the servers. The servers aggregate the gradient estimators, and update the model by using the aggregated gradients.

\begin{figure}[tb!]
\centering
\includegraphics[width=0.4\textwidth,height=3.4cm]{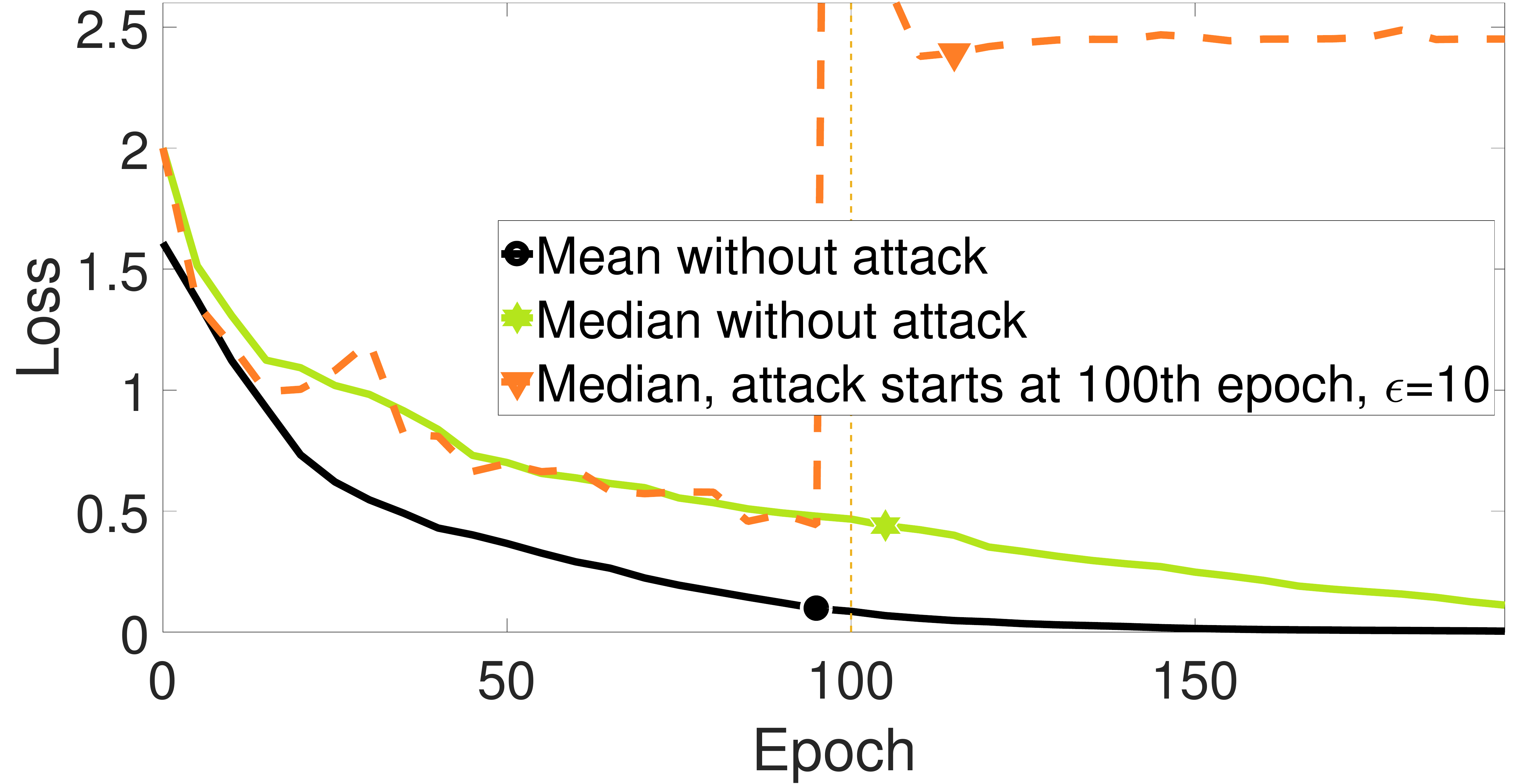}
\vskip 0.2cm
\includegraphics[width=0.4\textwidth,height=3.4cm]{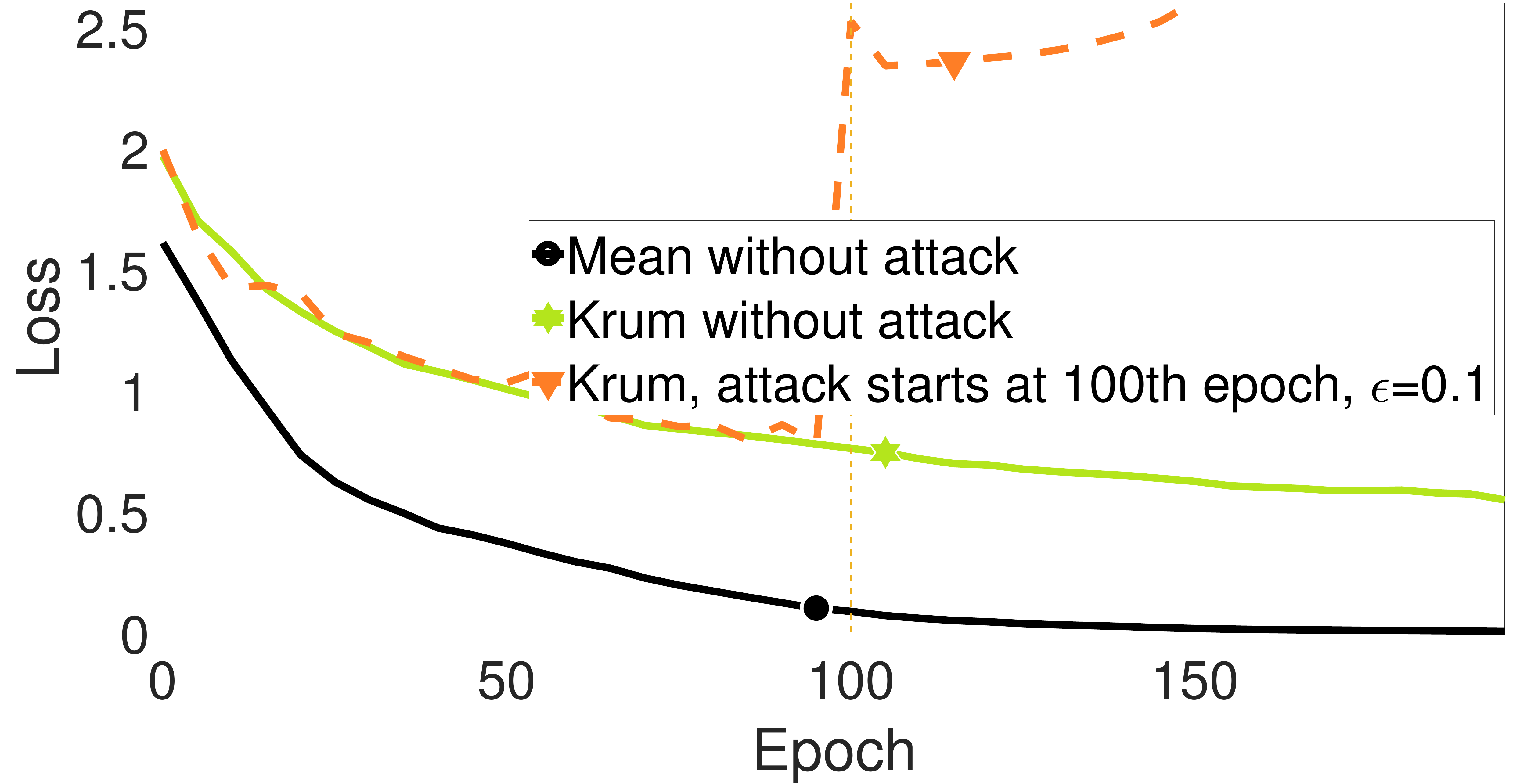}
\caption{Illustration of failed Byzantine-tolerant SGD. We execute distributed synchronous SGD on CIFAR-10 image classification, with 25 workers. Beginning from the 100th epoch, we attack the a system by replacing some workers with Byzantine workers. During the attack, 12 workers are Byzantine in the case of coordinate-wise median, and 11 workers are Byzantine in the case of Krum. The Byzantine workers push $- \epsilon g$ to the server, where $g$ is the true gradient.} 
\label{fig:observation}
\end{figure}

We consider  Byzantine failures at a subset of the worker nodes. 
Byzantine workers send arbitrary values instead of the gradient estimators to the server. Such Byzantine gradients are potentially adversarial, and this can result in convergence to sub-optimum models, or even lead to divergence. To make things worse, the Byzantine workers can spy on the information at any server or at any honest worker (omniscience). Byzantine gradients can thus be tailored to have similar variance and magnitude as the correct gradients, which makes them hard to distinguish. Additionally, in different iterations, different subsets of workers can behave in a  Byzantine manner, evading detection. Existing literature assumes that less than half of the workers are Byzantine in any iteration.

Compared to  traditional Byzantine tolerance in distributed systems~\citep{lynch1996distributed,avizienis2004basic,tanenbaum2007distributed,fischer1982impossibility}, Byzantine tolerance in distributed machine learning has unique properties and challenges. Traditional Byzantine tolerance attempts to reach consensus on correct values. However, machine learning algorithms do not need to reach consensus. Further, even non-Byzantine-tolerant machine learning algorithms can naturally tolerate some noise in the input and during execution~\citep{xing2016strategies}. Thus for distributed SGD, existing techniques for Byzantine-tolerant execution guarantee an upper-bound on the distance between the aggregated approximate gradient (under Byzantine workers) and the true gradient~\citep{blanchard2017machine,yin2018byzantine}. 


A deeper introspection reveals, however, that what really matters for gradient descent algorithms is the direction of the descent. As shown in Figure~\ref{fig:descent}, to let the gradient descent algorithm make progress, we need to guarantee that the direction of the aggregated vector agrees with the true gradient, i.e., the inner product between the aggregated vector and the true gradient must be non-negative. This can be violated by an attack that makes the inner product negative. 
We call this class of new attacks ``inner product manipulation attacks''.

\begin{figure}[htb!]
\centering
\includegraphics[width=0.28\textwidth]{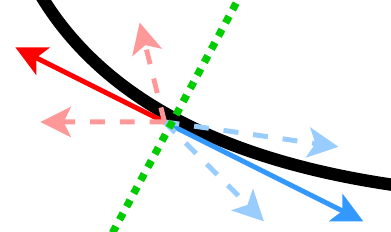}
\caption{Descent Direction. The blue arrows represent the directions which agree with the steepest descent direction (negative gradient). The red arrows represent directions which agree with the steepest ascent direction (gradient).}
\label{fig:descent}
\end{figure}

We observe that the bounded distance between the aggregated value and the true gradient guaranteed by existing techniques is not enough to defend distributed synchronous SGD against  inner product manipulation attacks. For example, for the coordinate-wise median, if we put all the Byzantine values on the opposite side of the true gradient, 
the inner product between the aggregated vector and the true gradient can be manipulated to be negative. 

In this paper, we study how inner product manipulation makes Byzantine-tolerant SGD vulnerable. We conduct case studies on coordinate-wise median~\citep{yin2018byzantine} and Krum~\citep{blanchard2017machine}.  Figure~\ref{fig:observation} gives a glimpse of how bad the effect of the attack can be. In a nutshell, creating gradients in the opposite direction with large magnitude crashes coordinate-wise median, while creating gradients in the opposite direction with small magnitude crashes Krum. We provide theoretical analysis as well as empirical results to validate these findings. 

Based on these results, we argue that there is a need to revise the definition of Byzantine tolerance in distributed SGD.  
We provide a new definition, and study its satisfaction on two prevailing robust distributed SGD algorithms, theoretically and empirically.
In summary, our contributions are:

\setitemize[0]{leftmargin=*}
\begin{itemize} 
\item We break two prevailing Byzantine tolerant SGD algorithms -- coordinate-wise median~\citep{yin2018byzantine} and Krum~\citep{blanchard2017machine} -- using a new class of attacks called inner production manipulation attacks. 
We theoretically prove that under certain conditions, we can backdoor these two algorithms, even when the assumptions and theorems presented in these papers are valid. 
\item We show how to design Byzantine gradients to compromise the robust aggregation rules. We conduct experiments to validate further.
\item Following our theoretical and empirical analysis, we propose a revised definition of Byzantine tolerance for distributed SGD. 
\end{itemize}

\section{RELATED WORK}

Robust estimators such as the median are well studied, and can naturally be applied to Byzantine tolerance. Coordinate-wise median is one approach that generalizes the median to high-dimensional vectors.  In \citet{yin2018byzantine}, statistical error rates are studied for using coordinate-wise median in distributed SGD. 

\citet{blanchard2017machine} propose Krum, which is not based on robust statistics. For each candidate gradient, Krum computes the local sum of squared Euclidean distances to the other candidates, and outputs the one with minimal sum.

In this paper, we focus on coordinate-wise median and Krum. There are other Byzantine-tolerant SGD algorithms. For example, Bulyan~\citep{guerraoui2018hidden} is built based on Krum, which potentially shares the same flaws.
DRACO~\citep{chen2018draco} uses coding theory to ensure robustness, and is different from the other Byzantine-tolerant SGD algorithms.

Recently, an increasing number of papers propose attack mechanisms to break the defense of machine learning in various scenarios. For example, \citet{athalye2018obfuscated} propose attack techniques using adversarial training data. \citet{bhagoji2018analyzing,bagdasaryan2018backdoor} break the defense of federated learning~\citep{mcmahan2016communication}. In this paper, we focus on attacking distributed synchronous SGD using adversarial gradients sent by Byzantine workers.

\section{PRELIMINARIES}

In this paper, we focus on distributed synchronous Stochastic Gradient Descent~(SGD) with Parameter Server~(PS). In this section, we formally introduce distributed synchronous SGD and the threat model of Byzantine failures.

\begin{figure}[htb!]
\centering
\includegraphics[width=0.49\textwidth,height=3.8cm]{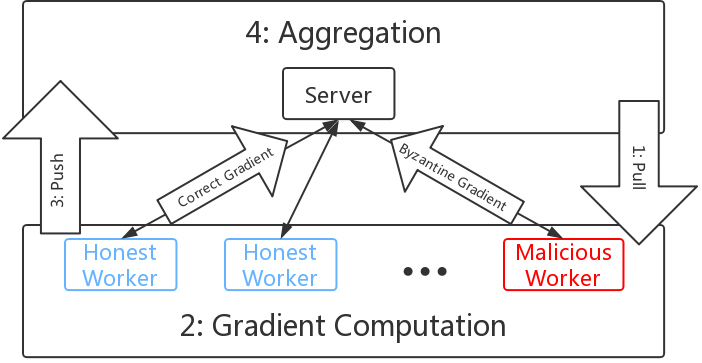}
\caption{Worker-Server Architecture}
\label{fig:ps}
\end{figure}

\subsection{STOCHASTIC GRADIENT DESCENT}

We consider the following optimization problem:
\begin{align*}
\min_{x \in \R^d} F(x), 
\end{align*}
where $F(x) = \E_{z \sim \mathcal{D}}[f(x; z)]$ is a differentiable function, $z$ is sampled from some unknown distribution $\mathcal{D}$, $d$ is the number of dimensions. We assume that there exists at least one minimizer of $F(x)$, which is denoted by $x^*$, where $\nabla F(x^*) = 0$.

This problem is solved in a distributed manner with $m$ workers. In each iteration, each worker will sample $n$ independent and identically distributed~(i.i.d.) data points from the distribution $\mathcal{D}$, and compute the gradient of the local empirical loss $F_i(x) = \frac{1}{n} \sum_{j=1}^n f(x; z^{i,j}), \forall i \in [m]$, where $z^{i,j}$ is the $j$th sampled data on the $i$th worker. The servers will collect and aggregate the gradients sent by the workers, and update the model as follows:
\begin{align*}
x^{t+1} = x^t - \gamma^t \aggr(\{\tilde{v}_i^t: i \in [m]\}),
\end{align*}
where $\aggr(\cdot)$ is an aggregation rule (e.g., averaging), and $\tilde{v}_i^t$ is the gradient sent by the $i$th worker, $\gamma^t$ is the learning rate in the $t^{\mbox{th}}$ iteration.
For an honest worker, 
$
\tilde{v}_i^t = \nabla F_i(x^t)
$
is an unbiased estimator such that $\E[\nabla F_i(x^t)] = \nabla F(x^t)$. 
When all the workers are honest, the most common choice of the aggregation rule $\aggr(\cdot)$ is averaging:
\begin{align*}
    \aggr(\{\tilde{v}_i^t: i \in [m]\}) = \frac{1}{m} \sum_{i \in [m]} \tilde{v}_i^t.
\end{align*}

The detailed algorithm of distributed synchronous SGD with aggregation rule $\aggr(\cdot)$ is shown in Algorithm~\ref{alg:sgd}.

\begin{algorithm}[htb!]
\caption{Distributed Synchronous SGD with Robust Aggregation}
\begin{algorithmic}
\vspace{0.2cm}
\STATE{\large\underline{\textbf{Server}}} 
\STATE $x^0 \leftarrow rand()$ \COMMENT{Initialization}
\FOR{$t = 0, \ldots, T$}
	\STATE Broadcast $x^{t}$ to all the workers
	\STATE Wait until all the gradients $\{\tilde{v}_i^t: i \in [m]\}$ arrive
	\STATE Compute $\bar{\tilde{v}}^t = \aggr(\{\tilde{v}_i^t: i \in [m]\})$ 
	\STATE Update the parameter $x^{t+1} \leftarrow x^{t} - \gamma^t \bar{\tilde{v}}^t$
\ENDFOR
\end{algorithmic}
\begin{algorithmic}
\vspace{0.2cm}
\STATE{\large\underline{\textbf{Worker}} $i = 1, \ldots, m$} 
\FOR{$t = 0, \ldots, T$}
	\STATE Receive $x^{t}$ from the server
	\STATE Draw the samples, compute, and send the gradient $v_i^t = \nabla F_i^t(x^{t})$ to the server
\ENDFOR
\end{algorithmic}
\label{alg:sgd}
\end{algorithm}

\subsection{THREAT MODEL}

In the Byzantine failure model, the gradients sent by malicious workers can take an arbitrary value:
\begin{align}
\label{equ:byz_grad}
\tilde{v}_i^t = 
\begin{cases}
*, & \mbox{if $i$th worker is Byzantine}, \\
\nabla F_i(x^t), & \mbox{otherwise,}
\end{cases}
\end{align}
where ``$*$" represents arbitrary values.

Formally, we define the threat model of Byzantine failure as follows.
\begin{definition} 
\label{def:byz}
(Threat Model~\citep{blanchard2017machine,Chen2017DistributedSM,yin2018byzantine}). In the $t^{\mbox{th}}$ iteration, let $\{v_i^t: i \in [m]\}$ be i.i.d. random vectors in $\R^d$, where $v_i^t = \nabla F_i(x^t)$. The set of correct vectors $\{v_i^t: i \in [m]\}$ is partially replaced by arbitrary vectors, which results in $\{\tilde{v}_i^t: i \in [m]\}$, as defined in Equation~(\ref{equ:byz_grad}). In other words, a correct gradient is $\nabla F_i(x^t)$, while a Byzantine gradient, marked as ``$*$", is assigned arbitrary value.  We assume that $q$ out of $m$ vectors are Byzantine, where $2q < m$. Furthermore, the indices of faulty workers can change across different iterations. If the failures are caused by attackers, the threat model includes the case where the attackers can collude.
\end{definition}



The notations used in this paper is summarized in Table~\ref{tbl:notations}.
\begin{table}[htb]
\vspace{-0.7cm}
\caption{Notations}
\label{tbl:notations}
\begin{center}
\begin{small}
\begin{tabular}{|l|l|}
\hline 
{\bf Notation}  & {\bf Description} \\ \hline
$m$    & Number of workers \\ \hline
$n$    & Minibatch size on each worker \\ \hline
$T$    & Number of iterations  \\ \hline
$[m]$    & Set of integers $\{1, \ldots, m \}$  \\ \hline
$q$    & Number of Byzantine workers  \\ \hline
$\gamma$    & Learning rate  \\ \hline
$x$    & Model parameters  \\ \hline
$\tilde{v}_i^t$    & Gradient sent by the $i$th worker \\
    &  in the $t^{\mbox{th}}$ iteration, potentially Byzantine \\ \hline
$v_i^t$   & Correct gradient produced by the $i$th worker  \\ 
    &  in the $t^{\mbox{th}}$ iteration \\ \hline
$\| \cdot \|$    & All the norms in this paper are $l_2$-norms  \\ \hline
$\ip{a}{b}$    & Inner product between $a$ and $b$  \\ \hline
\end{tabular}
\end{small}
\end{center}
\vspace{-0.7cm}
\end{table}

\section{DEFENSE TECHNIQUES}

In this section, we introduce two prevailing robust aggregation rules against Byzantine failures in distributed synchronous SGD: coordinate-wise median and Krum. For the remainder of this paper, we ignore the iteration superscript $t$ in $\tilde{v}_i^t$ and $v_i^t$ for convenience. 

\subsection{COORDINATE-WISE MEDIAN}
\begin{definition}(Coordinate-wise Median~\citep{yin2018byzantine})
\label{def:marmed}
We define the coordinate-wise median aggregation rule $\median(\cdot)$ as
\begin{align*}
    med = \median(\{\tilde{v}_i: i \in [m]\}),
\end{align*}
where for any $j\in[d]$, the $j$th dimension of $med$ is $med_j = median\left(\{(\tilde{v}_1)_j, \ldots, (\tilde{v}_m)_j\}\right)$, $(\tilde{v}_i)_j$ is the $j$th dimension of the vector $\tilde{v}_i$, $median(\cdot)$ is the one-dimensional median.  
\end{definition}


\subsection{KRUM}

\begin{definition}(Krum~\citep{blanchard2017machine})
\label{def:krum}
\begin{align*}
& \krum(\{\tilde{v}_i: i \in [m]\}) = \tilde{v}_k, \quad
k = \argmin_{i \in [m]} KR(\tilde{v}_i), \\
& KR(\tilde{v}_i) = \sum_{i \rightarrow j} \| \tilde{v}_i - \tilde{v}_j \|^2,
\end{align*}
where $i \rightarrow j$ are the indices of the $m-q-2$ nearest neighbours of $\tilde{v}_i$ in $\{\tilde{v}_j: j \in [m], i \neq j\}$ as measured by squared Euclidean distance.
\end{definition}

For convenience, we refer to the coordinate-wise median and Krum as \texttt{Median} and \texttt{Krum}.

\section{ATTACK TECHNIQUES}

In this section, we revise the definition of Byzantine tolerance in distributed synchronous SGD. Then, we theoretically analyze the Byzantine tolerance of coordinate-wise median and Krum, and show that under certain conditions, these two robust aggregation rules are no longer Byzantine-tolerant.

\subsection{INNER PRODUCT MANIPULATION}

In the previous work on Byzantine-tolerant SGD algorithms, most of the robust aggregation rules only guarantee that the robust estimator is not arbitrarily far away from the mean of the correct gradients. In other words, the distance between the robust estimator and the correct mean is upper-bounded. However, for gradient descent algorithms, to guarantee the descent of the loss, the inner product between the true gradient and the robust estimator must be non-negative: 
\begin{align*}
    \ip{\nabla F(x)}{\aggr(\{\tilde{v}_i: i \in [m]\})} \geq 0,
\end{align*}
so that at least the loss will not increase in expectation. In particular, bounded distance is not enough to guarantee robustness, if the attackers manipulate the Byzantine gradients and make the inner product negative.

The intuition underlying the inner product manipulation attack is that, when gradient descent algorithm converges, the gradient $\nabla F(x^t)$ approaches $0$. Thus, even if the distance between the robust estimator and the correct mean is bounded, it is still possible to manipulate their inner product to be negative, especially when the upper bound of such distance is large. 

We formally define a revised version of Byzantine tolerance for distributed synchronous SGD~(DSSGD-Byzantine tolerance):
\begin{definition} (DSSGD-Byzantine Tolerance)
Without loss of generality, suppose that in a specific iteration, the server receives $(m-q)$ correct gradients $\mathcal{V} = \{v_1, \ldots, v_{m-q}\}$ and $q$ Byzantine gradients $\mathcal{U} = \{u_1, \ldots, u_q\}$. We assume that the correct gradients have the same expectation $\E[v_i] = g, \forall i \in [m-q]$. An aggregation rule $\aggr(\cdot)$ is said to be DSSGD-Byzantine-tolerant if 
\begin{align*}
    \ip{g}{\quad \E\left[ \aggr(\mathcal{V} \cup \mathcal{U}) \right]} \geq 0.
\end{align*}
\end{definition}

With the revised definition, now we theoretically analyze the DSSGD-Byzantine tolerance of coordinate-wise median and Krum.

\begin{remark}
Note that we do not argue that the theoretical guarantees in the previous work are wrong. Instead, our claim is that the theoretical guarantees on the bounded distances are not enough to secure distributed synchronous SGD. In particular, DSSGD-Byzantine tolerance is different from the Byzantine tolerance proposed in previous work. 

\end{remark}

\subsection{COORDINATE-WISE MEDIAN}

The following theorem shows that under certain conditions, \texttt{Median} is not DSSGD-Byzantine-tolerant.
\begin{theorem} \label{thm:median}
We consider the worst case where $m-2q = 1$. The server receives $(m-q)$ correct gradients $\mathcal{V} = \{v_1, \ldots, v_{m-q}\}$ and $q$ Byzantine gradients $\mathcal{U} = \{u_1, \ldots, u_q\}$. We assume that the stochastic gradients have identical expectation $\E[v_i] = g, \forall i \in [m-q]$, and non-zero coordinate-wise variance $\E[ \left( (v_i)_j - g_j \right)^2 ] \geq \sigma^2, \forall i \in [m-q], j \in [d]$, where $(v_i)_j$ is the $j$th coordinate of $v_i$, and $g_j$ is the $j$th coordinate of $g$. When $\max_{j \in [d]} | g_j | < \frac{\sigma}{\sqrt{m-q-1}}$, there exist Byzantine gradients $\mathcal{U} = \{u_1, \ldots, u_q\}$ such that
\begin{align*}
    \ip{g}{\quad \E\left[ \median(\mathcal{V} \cup \mathcal{U}) \right]} < 0.
\end{align*}
\end{theorem}

\begin{proof} (sketch)
Since median is independently taken in each coordinate, it is sufficient to prove Byzantine vulnerability for one coordinate or scalars. Thus, for convenience, with a little bit abuse of notation, we suppose that the correct gradients $\mathcal{V} = \{v_1, \ldots, v_{m-q}\}$ and $q$ Byzantine gradients $\mathcal{U} = \{u_1, \ldots, u_q\}$ are all scalars. We only need to show that under certain attacks, the aggregated value $\median(\mathcal{V} \cup \mathcal{U})$ has a different sign than $\sum_{i \in [m-q]} v_i$. 

Without loss of generality, we assume that $g = \frac{1}{m-1} \sum_{i \in [m-q]} \E[v_i] > 0$ (the mirror case can be easily proved with a similar procedure). The Byzantine gradients are all assigned negative value: $u_i < 0, \forall i \in [q]$. Furthermore, we make the Byzantine gradients small enough such that $u_i < \min(\mathcal{V}), \forall i \in [q]$.

By sorting the correct gradients, we can define the sequence $\{v_{1:m-q}, \ldots, v_{m-q:m-q}\}$, where $v_{i:m-q}$ is the $i$th smallest element in $\{v_1, \ldots, v_{m-q}\}$:
\begin{align*}
    v_{1:m-q} \leq v_{2:m-q} \leq \cdots \leq v_{m-q:m-q}.
\end{align*}
We also define the expectation of the $i$th smallest element: $\mu_{i:m-q} = \E[ v_{i:m-q} ]$.

Then, it is easy to check that $\median(\mathcal{V} \cup \mathcal{U}) = v_{1:m-q}$, and $\E \left[ \median(\mathcal{V} \cup \mathcal{U}) \right] = \mu_{1:m-q}$.

Using Theorem 1(b) from \cite{hawkins1971bounds} (equiv. 9(a) from \cite{arnold1979bounds}), we have 
\begin{align*}
\mu_{1:m-q} \leq g - \frac{\sigma}{\sqrt{m-q-1}}.
\end{align*}

Thus, when $g < \frac{\sigma}{\sqrt{m-q-1}}$, $\E \left[ \median(\mathcal{V} \cup \mathcal{U}) \right]$ is negative.

\end{proof}

\begin{remark}
When gradient descent converges, the expectation of the gradient $g$ approaches $0$. Furthermore, since the gradient produced by the correct workers are stochastic, the variance always exists. Thus, eventually, the condition $\max_{j \in [d]} | g_j | < \frac{\sigma}{\sqrt{m-q-1}}$ will be satisfied. To make things worse, the closer SGD approaches a critical point, the less likely the coordinate-wise median is DSSGD-Byzantine-tolerant.
\end{remark}

\begin{remark} \label{rmk:median}
The proof of Theorem~\ref{thm:median} provides the intuition of constructing adversarial gradients for the attackers. In practice, in each coordinate, the attackers only need to guarantee that all the Byzantine values are much smaller than the smallest correct value if the correct expectation is positive, or much larger than the largest correct value if the correct expectation is negative. Then, hopefully, if the variance is large enough, the smallest/largest value has the opposite sign to the correct expectation. Then, the attackers can successfully manipulate the aggregated value into the opposite direction to the correct expectation.
\end{remark}

\subsubsection{Toy Example}

We provide an 1-dimensional toy example to illustrate how easily \texttt{Median} can fail. Suppose there are $3$ correct gradients $\mathcal{V} = \{-0.1, 0.1, 0.3\}$ with the mean $0.1$, and $2$ Byzantine gradient $\mathcal{U} = \{-4, -2\}$ with the negative mean $-3$. According to Definition~\ref{def:marmed}, it is easy to check that $\median(\mathcal{U} \cap \mathcal{V})  = -0.1$, which means that \texttt{Median} produces a value with the opposite sign of the mean of the correct gradients.

\subsection{KRUM}

The following theorem proves that under certain conditions, Krum is not DSSGD-Byzantine-tolerant. Note that Krum requires that $m-2q > 2$. 
\begin{theorem} \label{thm:krum}
We consider the worst case where $m-2q = 3$. The server receives $(m-q)$ correct gradients $\mathcal{V} = \{v_1, \ldots, v_{m-q}\}$ and $q$ Byzantine gradients $\mathcal{U} = \{u_1, \ldots, u_q\}$. We assume that the stochastic gradients have identical expectation $\E[v_i] = g, \forall i \in [m-q]$. We define the mean of the correct gradients $\bar{v} = \frac{1}{m-q} \sum_{i \in [m-q]} v_i$. We assume that the correct gradients are bounded by $\|v_i - \bar{v}\|^2 \leq \|\bar{v}\|^2, \forall i \in [m-q]$. Furthermore, we assume that $v_i \neq v_j, \forall i \neq j, i,j \in [m-q]$, and $\exists \beta$ such that $\|v_i - v_j\|^2 \geq \beta^2, \forall i \neq j, i,j \in [m-q]$. We take $u_1 = u_2 = \cdots = u_q = - \epsilon \bar{v}$, where $\epsilon$ is a small positive constant value such that $\epsilon^2 \|\bar{v}\|^2 \leq \beta^2$. When $(m-q)$ is large enough: $m-q > \frac{2(\epsilon+2)^2}{\epsilon^2} + 2$, we have
\begin{align*}
    \ip{g}{\quad \E\left[ \krum(\mathcal{V} \cup \mathcal{U}) \right]} < 0.
\end{align*}
\end{theorem}

\begin{proof} (sketch)
For $\forall u \in \mathcal{U}$, $u = -\epsilon \bar{v}$, where $\bar{v} = \frac{1}{m-1} \sum_{i \in [m-q]} v_i$.

Since any $u \in \mathcal{U}$ is identical, the nearest $(m-q-4)$ neighbours of $u$ must belong to $\mathcal{U}$. The remaining $(m-q-2) - (m-q-4) = 2$ nearest neighbours must belong to the set of correct gradients $\mathcal{V}$. Thus, we have
\begin{align*}
    KR(u) \leq 2 \| \bar{v} + \bar{v} + \epsilon \bar{v} \|^2 = 2(\epsilon+2)^2 \|\bar{v}\|^2.
\end{align*}
For the correct gradients $\forall v \in \mathcal{V}$, there are two cases:
\begin{itemize}
    \item \textbf{Case 1:} There are some $u \in \mathcal{U}$ which belong to the $(m-q-2)$ nearest neighbours of $v$. 
    
    Suppose there are $a_1$ nearest neighbours in $\mathcal{V}$ and $a_2$ nearest neighbours in $\mathcal{U}$, where $a_1 + a_2 = m-q-2$. 
    Since the correct gradients are bounded by $\|v_i - \bar{v}\|^2 \leq \|\bar{v}\|^2, \forall i \in [m-q]$, it is easy to check that $\|v-u\|^2 \geq \epsilon^2 \|\bar{v}\|^2$.
    Thus, we have 
    \begin{align*}
        KR(v) \geq a_1 \beta^2 + a_2 \|v-u\|^2 \geq (m-q-2) \epsilon^2 \|\bar{v}\|^2.
    \end{align*}
    
    \item \textbf{Case 2:} There are no $u \in \mathcal{U}$ which belong to the $(m-q-2)$ nearest neighbours of $v$. 
    Thus, we have
    \begin{align*}
        KR(v) \geq (m-q-2) \beta^2 \geq (m-q-2) \epsilon^2 \|\bar{v}\|^2.
    \end{align*}
\end{itemize}
In both cases, we have 
$
    KR(v) \geq (m-q-2) \epsilon^2 \|\bar{v}\|^2.
$
Thus, when $(m-q)$ is large enough: $m-q > \frac{2(\epsilon+2)^2}{\epsilon^2} + 2$, we have
\begin{align*}
    KR(u) &\leq  2(\epsilon+2)^2 \|\bar{v}\|^2 <  (m-q-2) \epsilon^2 \|\bar{v}\|^2 \\
    &\leq KR(v). 
\end{align*}
As a result, $\krum(\mathcal{V} \cap \mathcal{U}) = u = -\epsilon \bar{v}$. Thus, $\E\left[ \krum(\mathcal{V} \cap \mathcal{U}) \right] = -\epsilon g$.

\end{proof}

\begin{remark}
In the theorem above, we assume that all the correct gradients are inside a Euclidean ball centered at their mean: $\|v_i - \bar{v}\|^2 \leq \|\bar{v}\|^2, \forall i \in [m-q]$. Such assumption can not always be satisfied, but it is reasonable that the random samples are sometimes inside such a Euclidean ball, if the variance is not too large. On the other hand, we assume that the pair-wise distances between the correct gradients are lower-bounded by $\beta > 0$. Almost surely, such $\beta$ exists, no matter how small it is. Note that the Byzantine attackers are supposed to be omniscient. Thus, the attackers can spy on the honest workers, and obtain $\mathcal{V}$ and $\beta$. Then, the attackers can choose an $\epsilon$ such that $\epsilon^2 \|\bar{v}\|^2 \leq \beta^2$. Finally, we only need the number of workers to be large enough, so that $m-q > \frac{2(\epsilon+2)^2}{\epsilon^2} + 2$.
\end{remark}

\begin{remark} \label{rmk:krum}
The proof of Theorem~\ref{thm:krum} provides the intuition of constructing adversarial gradients for the attackers. In practice, the attackers only need to assign $\frac{\epsilon}{m-q} \sum_{i \in [m-q]} v_i$ to all the Byzantine gradients, with an $\epsilon > 0$ small enough.
\end{remark}

\begin{remark}
Note that in \citet{blanchard2017machine}, \texttt{Krum} requires the assumption that $c \sigma < \|g\|$ for convergence, where $c$ is a general constant, $\sigma$ is the maximal variance of the gradients, and $g$ is the gradient in expectation. Note that $\|g\| \rightarrow 0$ when SGD converges to a critical point. Thus, such an assumption is never guaranteed to be satisfied, if the variance is non-zero. Furthermore, the better SGD converges, the less likely such an assumption can be satisfied.
\end{remark}

\subsubsection{Toy Example}

Note that the assumptions made in Theorem~\ref{thm:krum} are sufficient but not necessary conditions of the DSSGD-Byzantine vulnerability of \texttt{Krum}. In practice, it can be easier to find an $\epsilon$ that crashes \texttt{Krum}, especially for 1-dimensional cases. 

We provide an 1-dimensional toy example to show how easily \texttt{Krum} can fail. Suppose there are $6$ correct gradients $\mathcal{V} = \{0,0.02,0.14,0.26,0.38,0.5\}$ with the mean $0.2167$, and $3$ Byzantine gradient $\mathcal{U} = \{-0.1, -0.1, -0.1\}$ with the negative mean $-0.1$. According to Definition~\ref{def:krum}, the corresponding function values $KR(\cdot)$ of $\mathcal{U} \cap \mathcal{V} = \{-0.1, -0.1, -0.1, 0,0.02,0.14,0.26,0.38,0.5\}$ are $\{0.0244,0.0244,0.0244,0.0304,0.0436,0.1060, 0.1440,\allowbreak 0.2160,0.4320\}$. Thus, $\krum(\mathcal{U} \cap \mathcal{V})  = -0.1$, which means that \texttt{Krum} chooses the Byzantine gradient with the opposite sign of the mean of the correct gradients.

\begin{figure*}[htb!]
\centering
\subfigure[Top-1 Accuracy on Testing Set, $\epsilon=10$]{\includegraphics[width=0.49\textwidth,height=4.4cm]{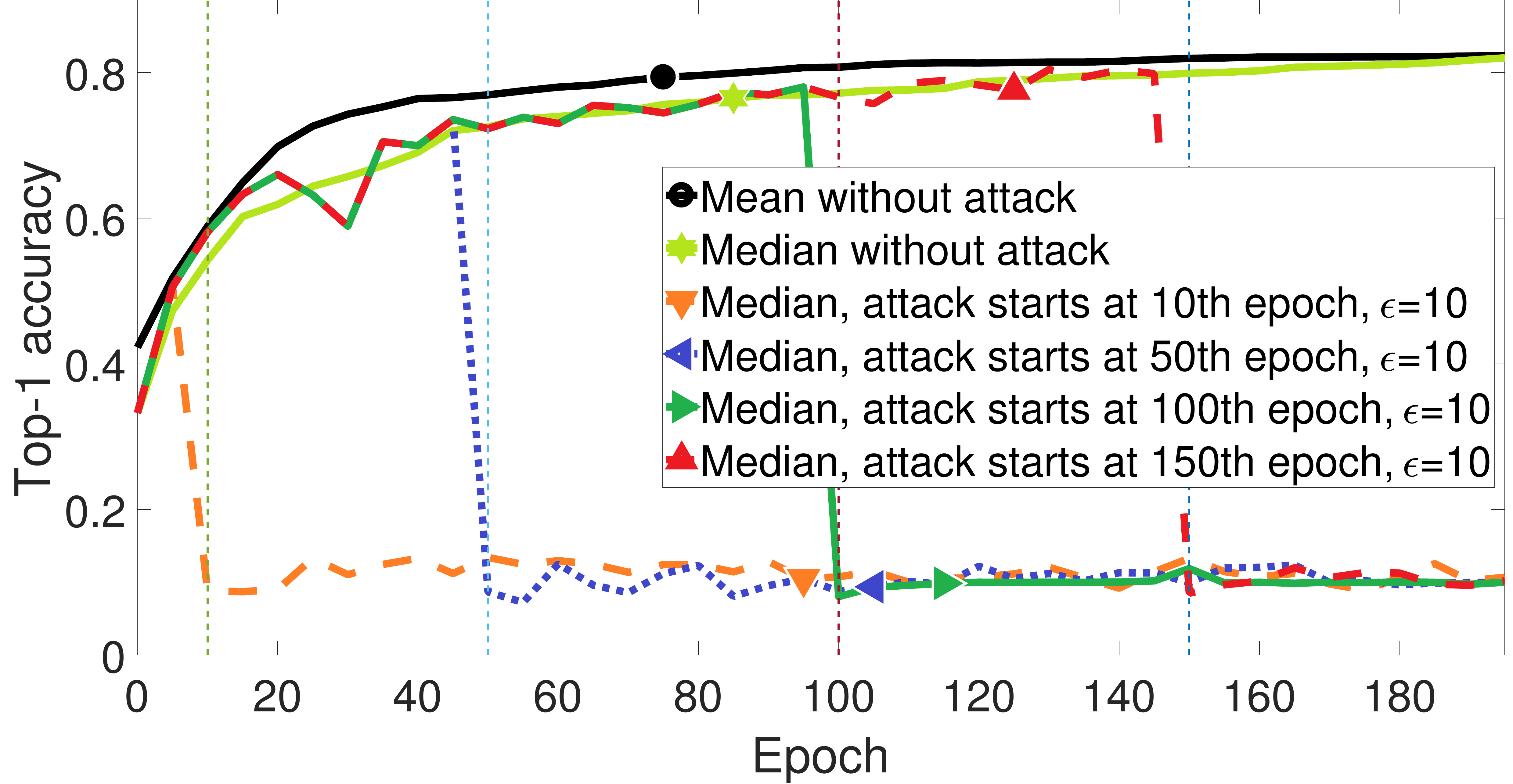}}
\subfigure[Cross Entropy on Training Set, $\epsilon=10$]{\includegraphics[width=0.49\textwidth,height=4.4cm]{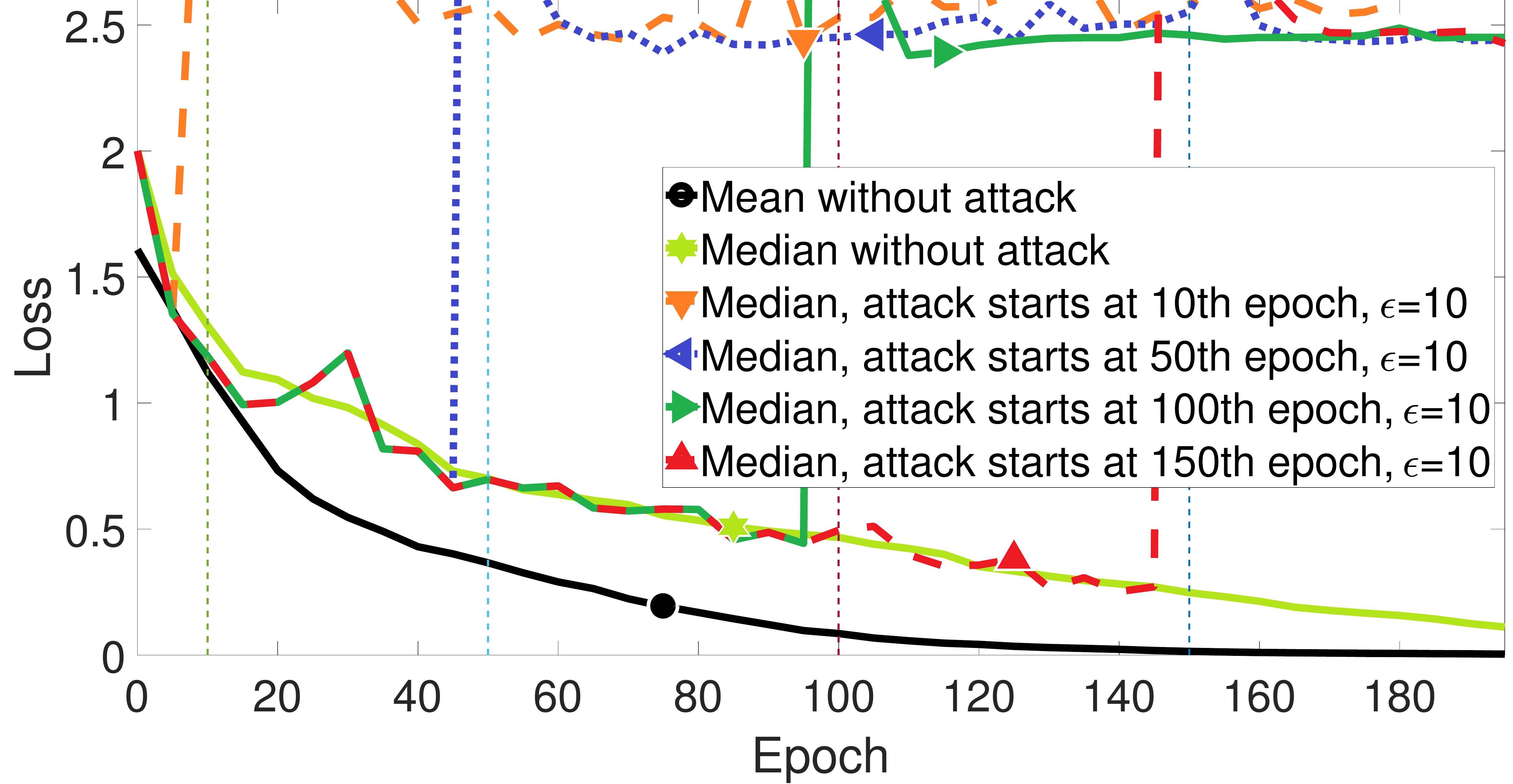}}
\subfigure[Top-1 Accuracy on Testing Set, $\epsilon=0.1$]{\includegraphics[width=0.49\textwidth,height=4.4cm]{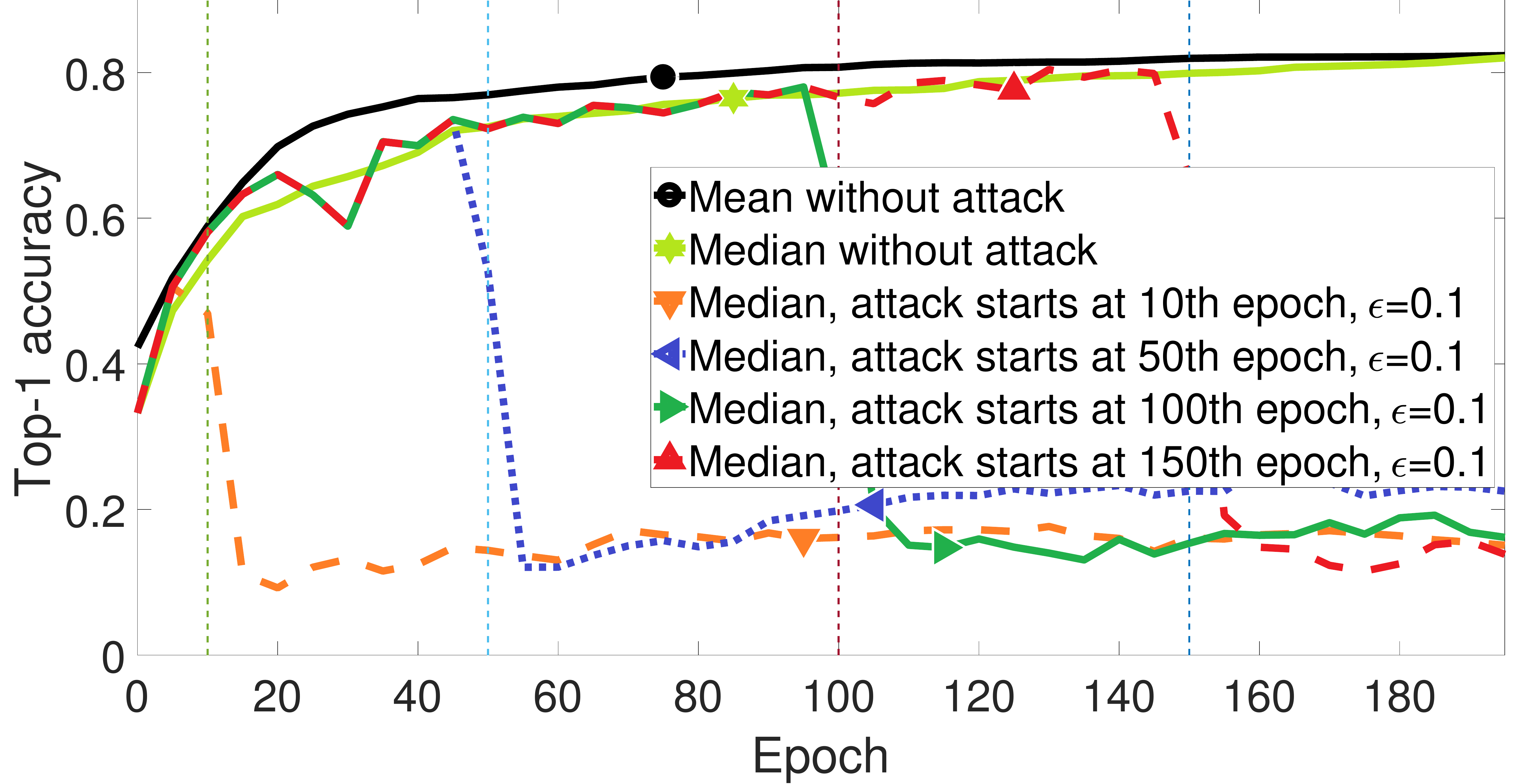}}
\subfigure[Cross Entropy on Training Set, $\epsilon=0.1$]{\includegraphics[width=0.49\textwidth,height=4.4cm]{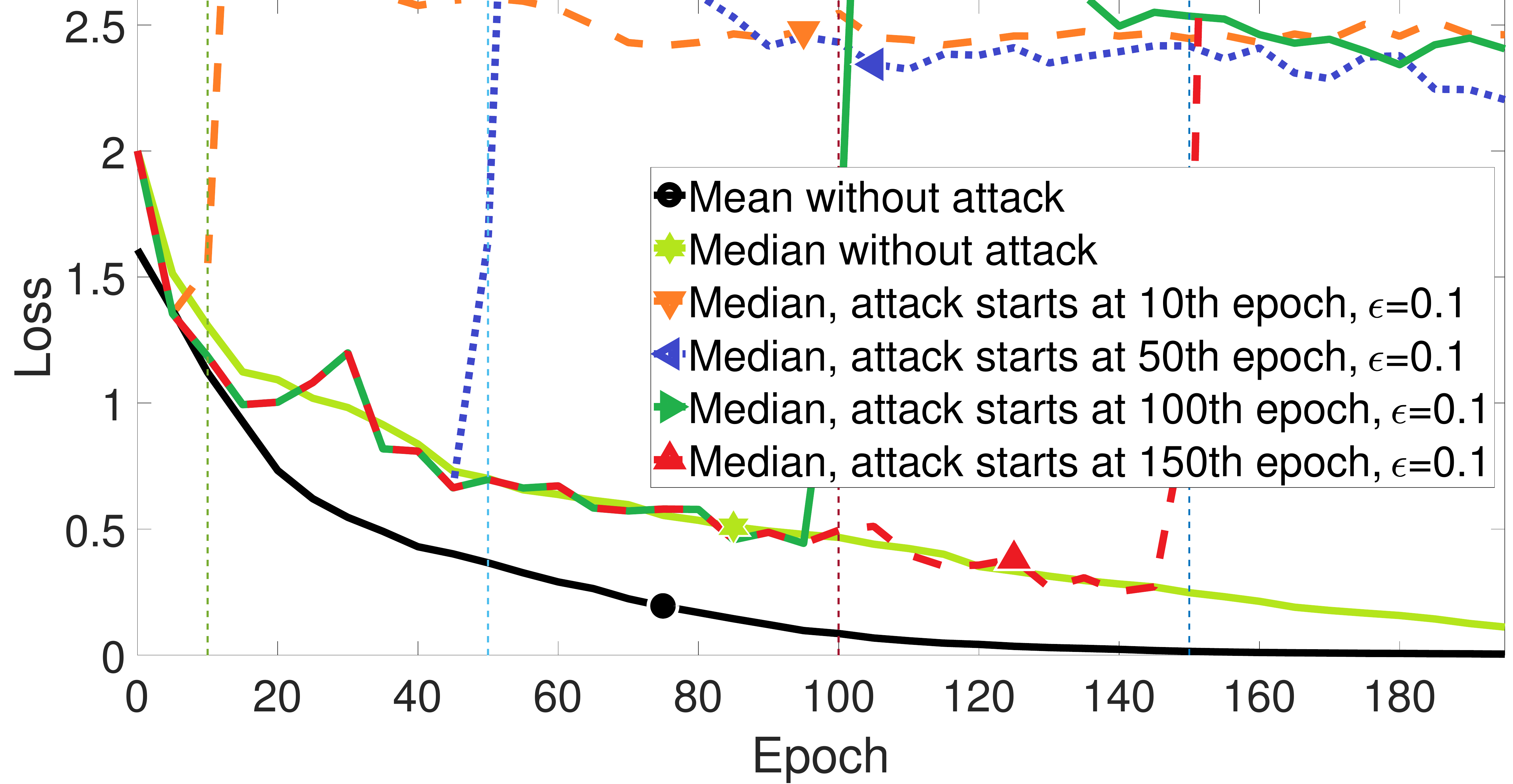}}
\subfigure[Top-1 Accuracy on Testing Set, $\epsilon=0$]{\includegraphics[width=0.49\textwidth,height=4.4cm]{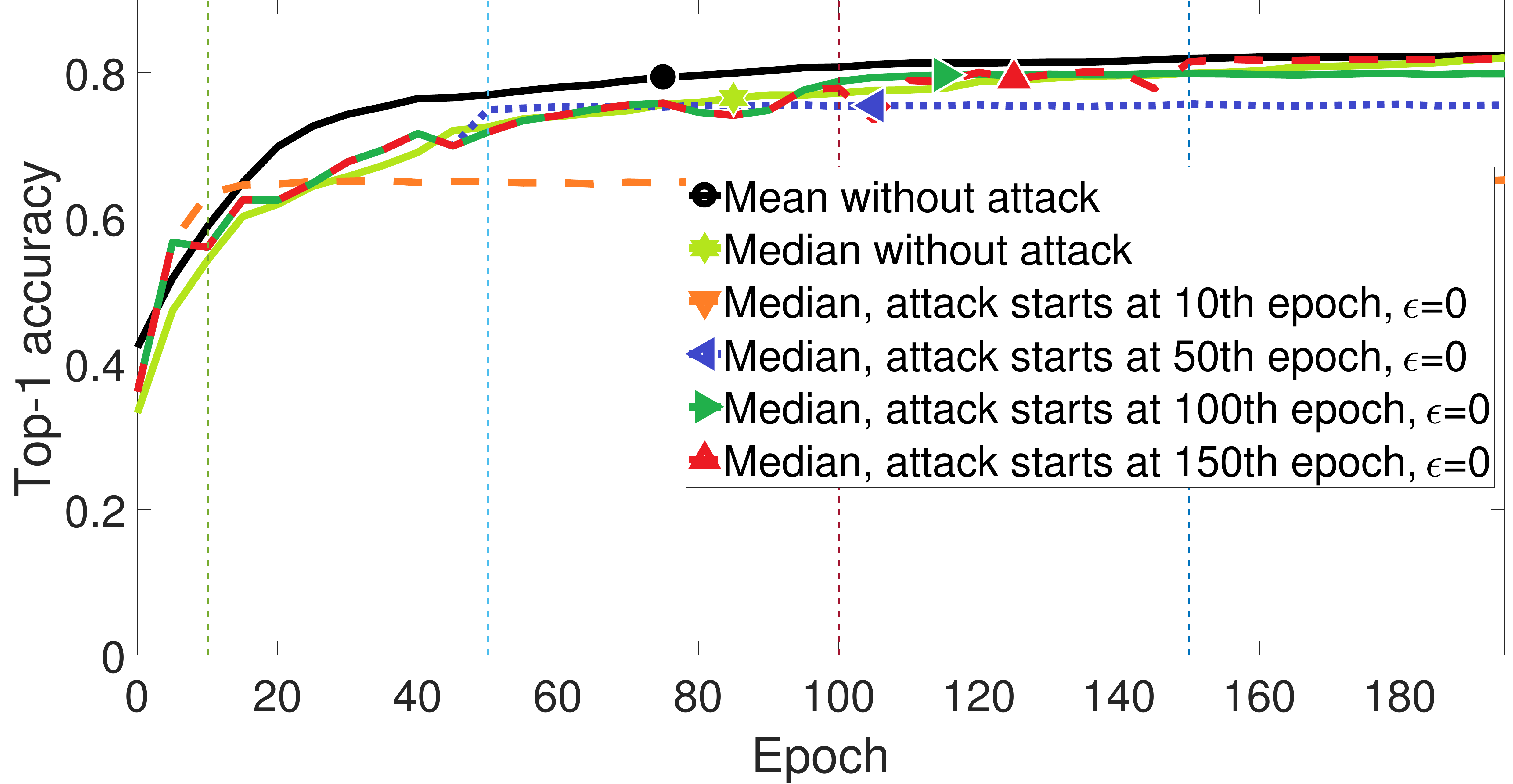}}
\subfigure[Cross Entropy on Training Set, $\epsilon=0$]{\includegraphics[width=0.49\textwidth,height=4.4cm]{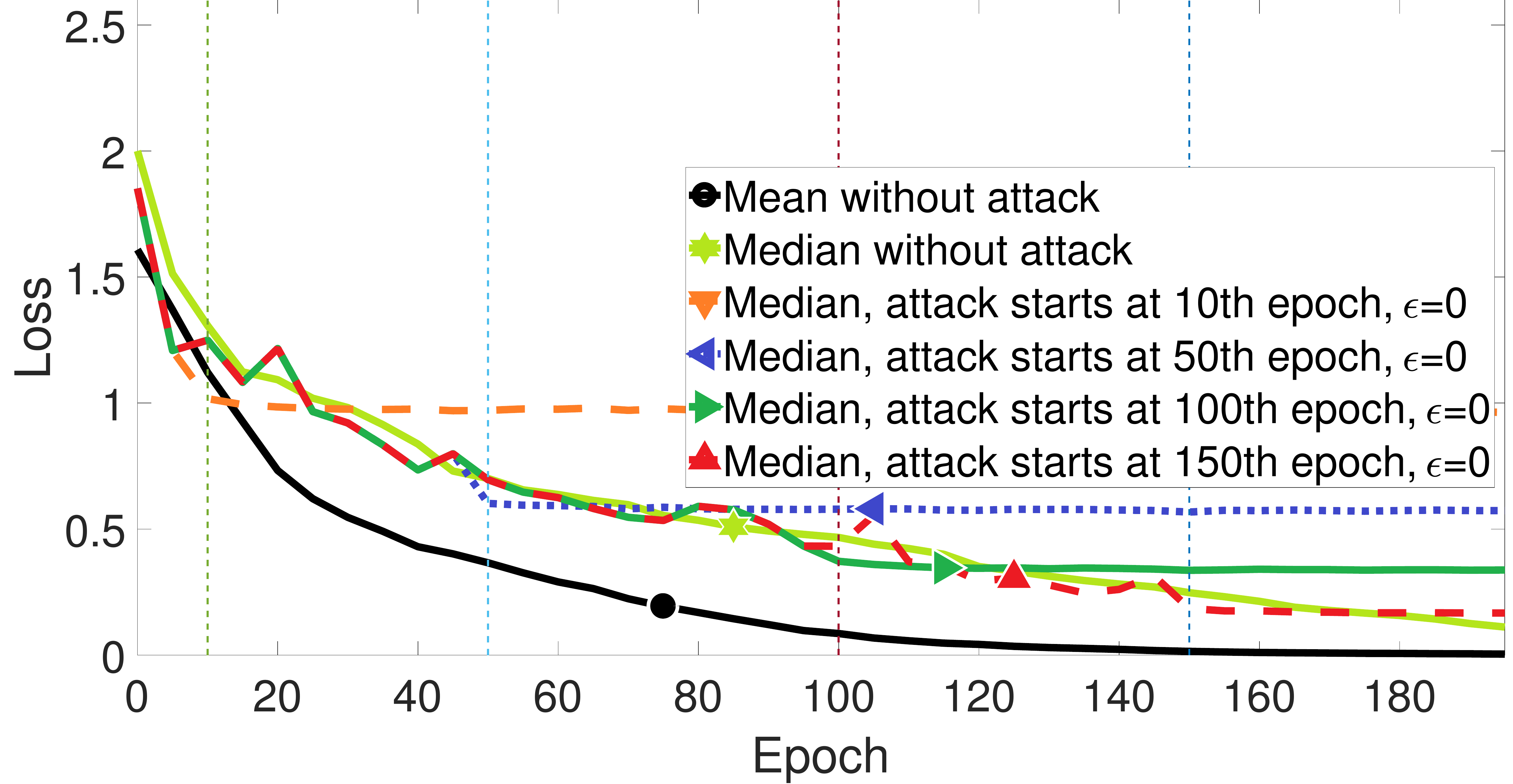}}
\subfigure[Top-1 Accuracy on Testing Set, $\epsilon=-10$]{\includegraphics[width=0.49\textwidth,height=4.4cm]{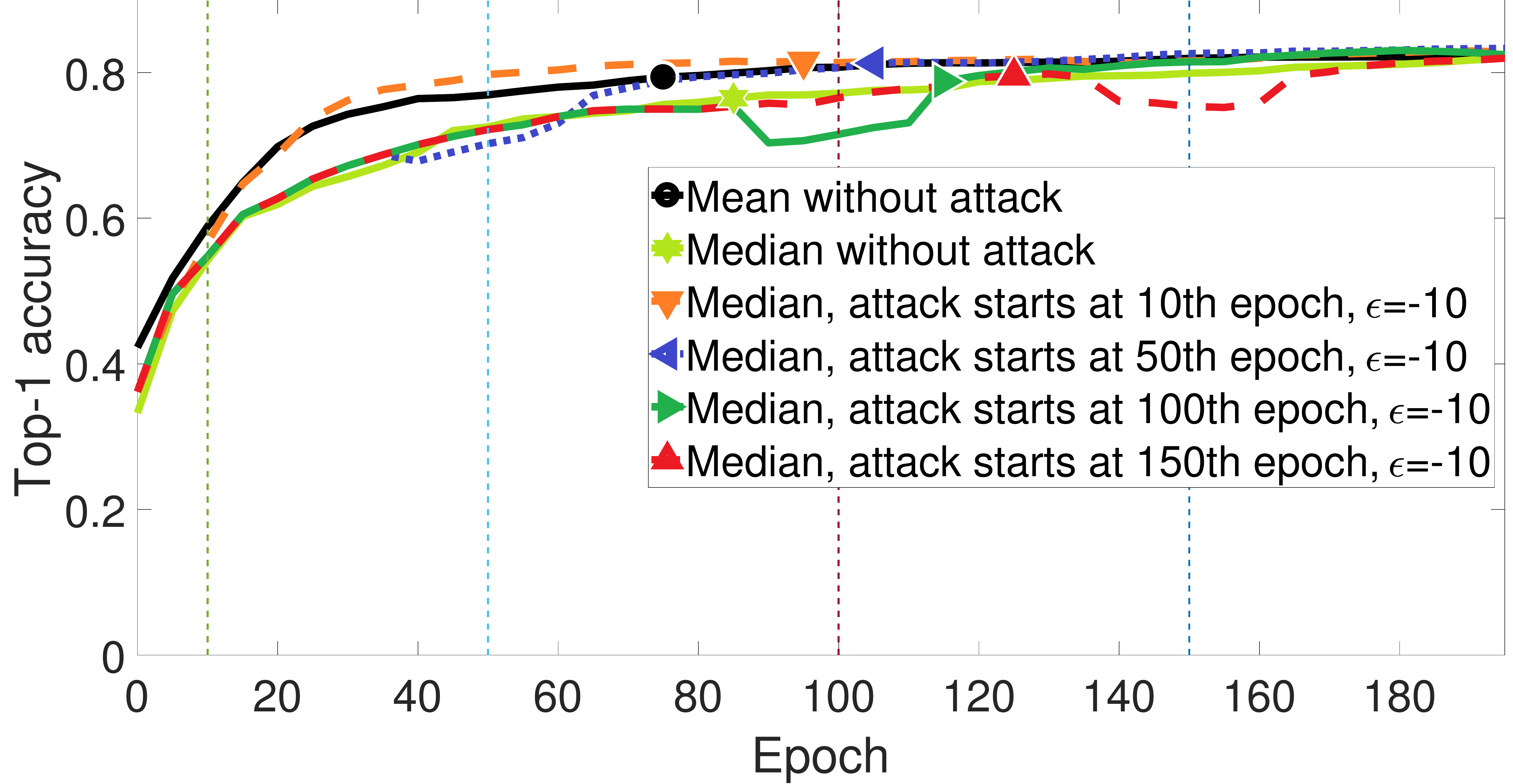}}
\subfigure[Cross Entropy on Training Set, $\epsilon=-10$]{\includegraphics[width=0.49\textwidth,height=4.4cm]{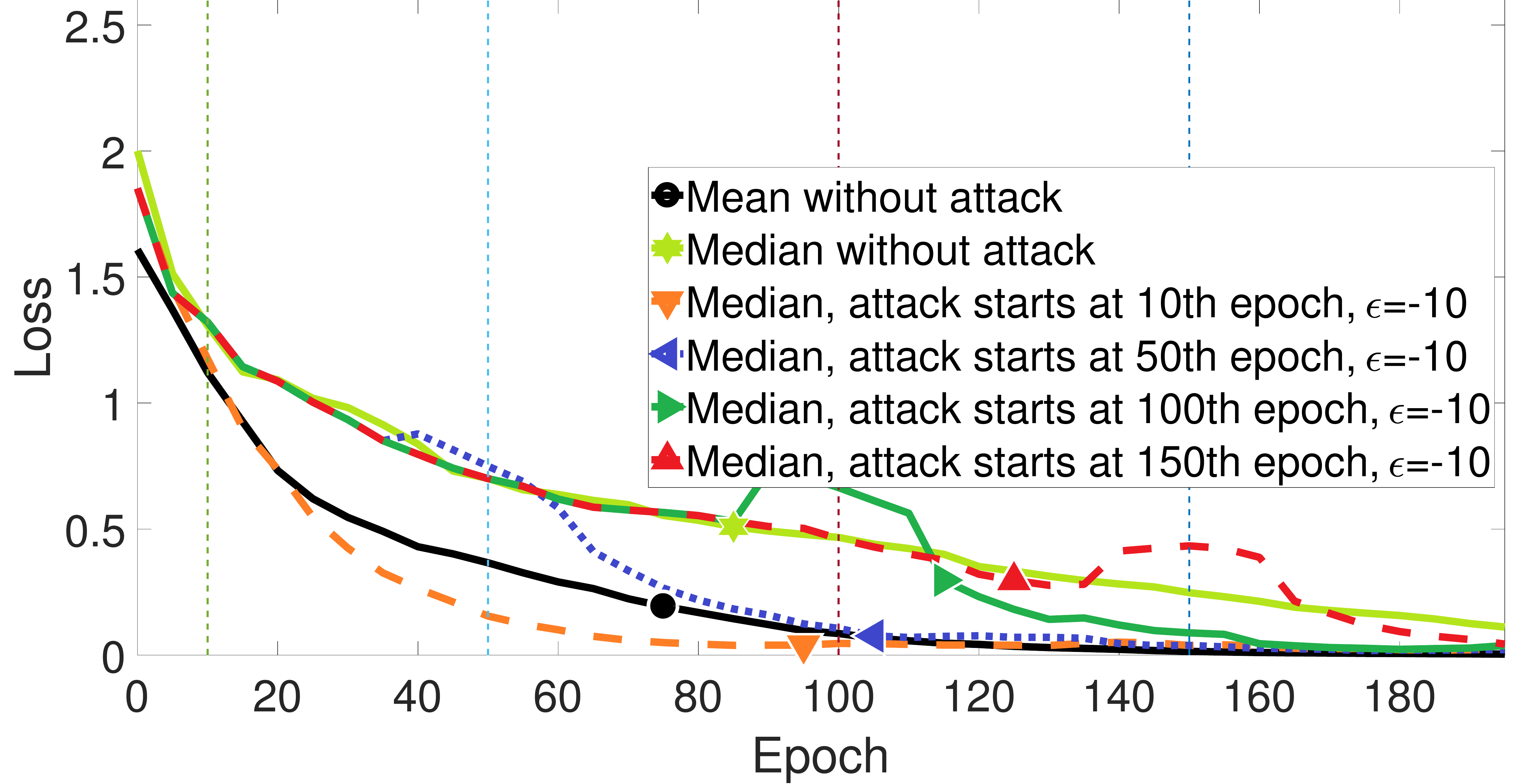}}
\caption{Convergence on training set, using \texttt{Median} as aggregation rule. $\epsilon \in \{10, 0.1, 0, -10\}$.}
\label{fig:median}
\end{figure*}

\begin{figure*}[htb!]
\centering
\subfigure[Top-1 Accuracy on Testing Set, $\epsilon=0.1$]{\includegraphics[width=0.49\textwidth,height=4.4cm]{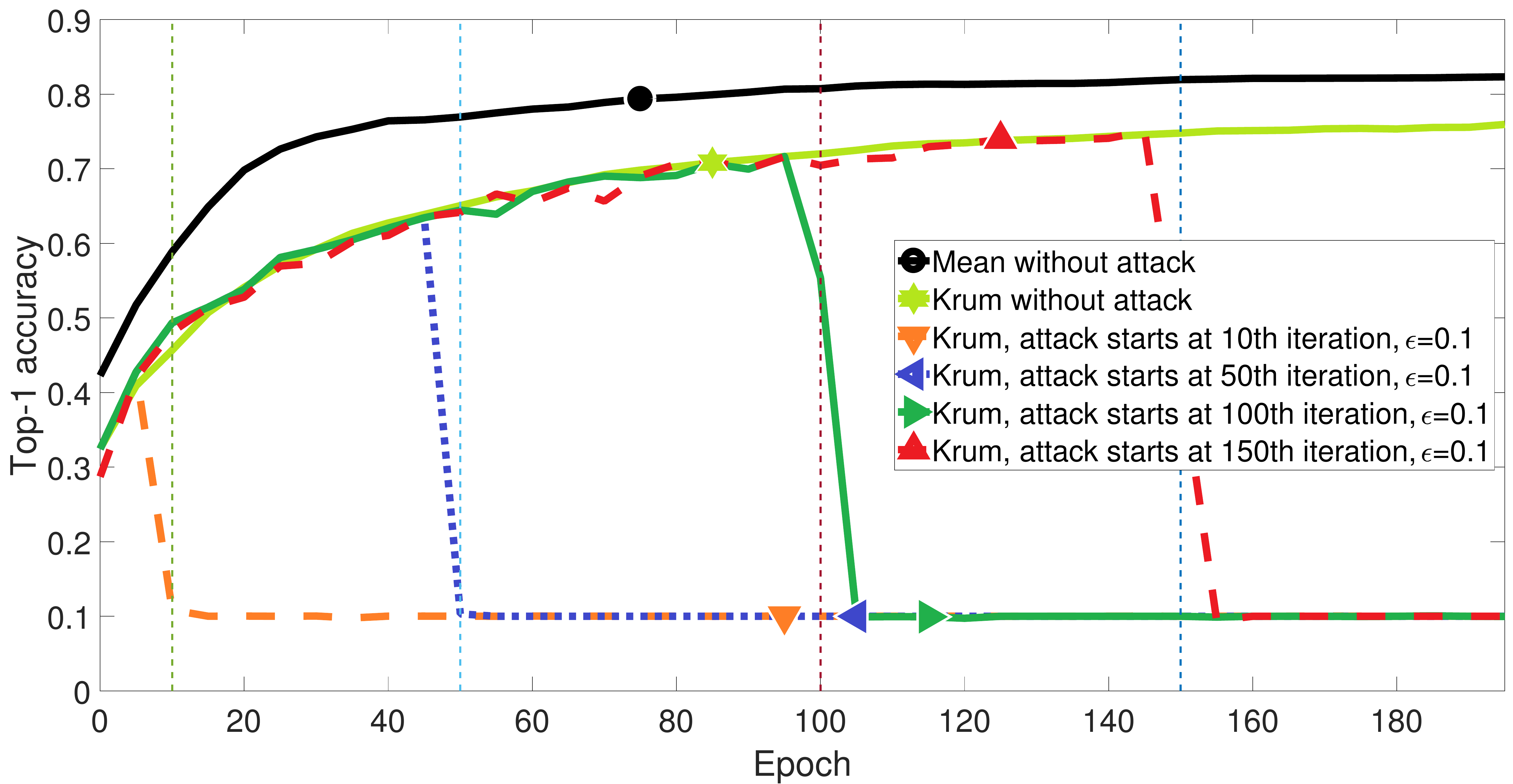}}
\subfigure[Cross Entropy on Training Set, $\epsilon=0.1$]{\includegraphics[width=0.49\textwidth,height=4.4cm]{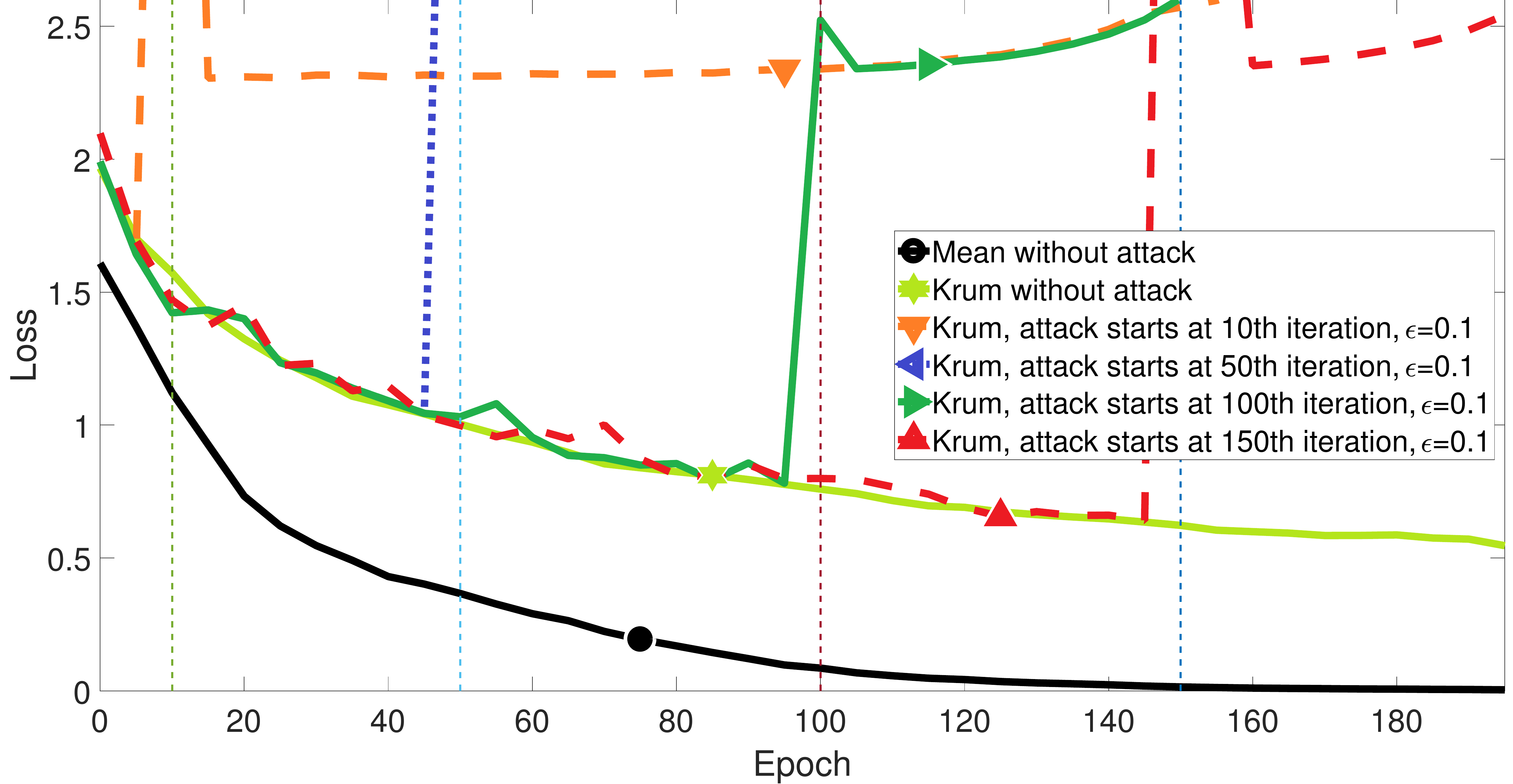}}
\subfigure[Top-1 Accuracy on Testing Set, $\epsilon=0.5$]{\includegraphics[width=0.49\textwidth,height=4.4cm]{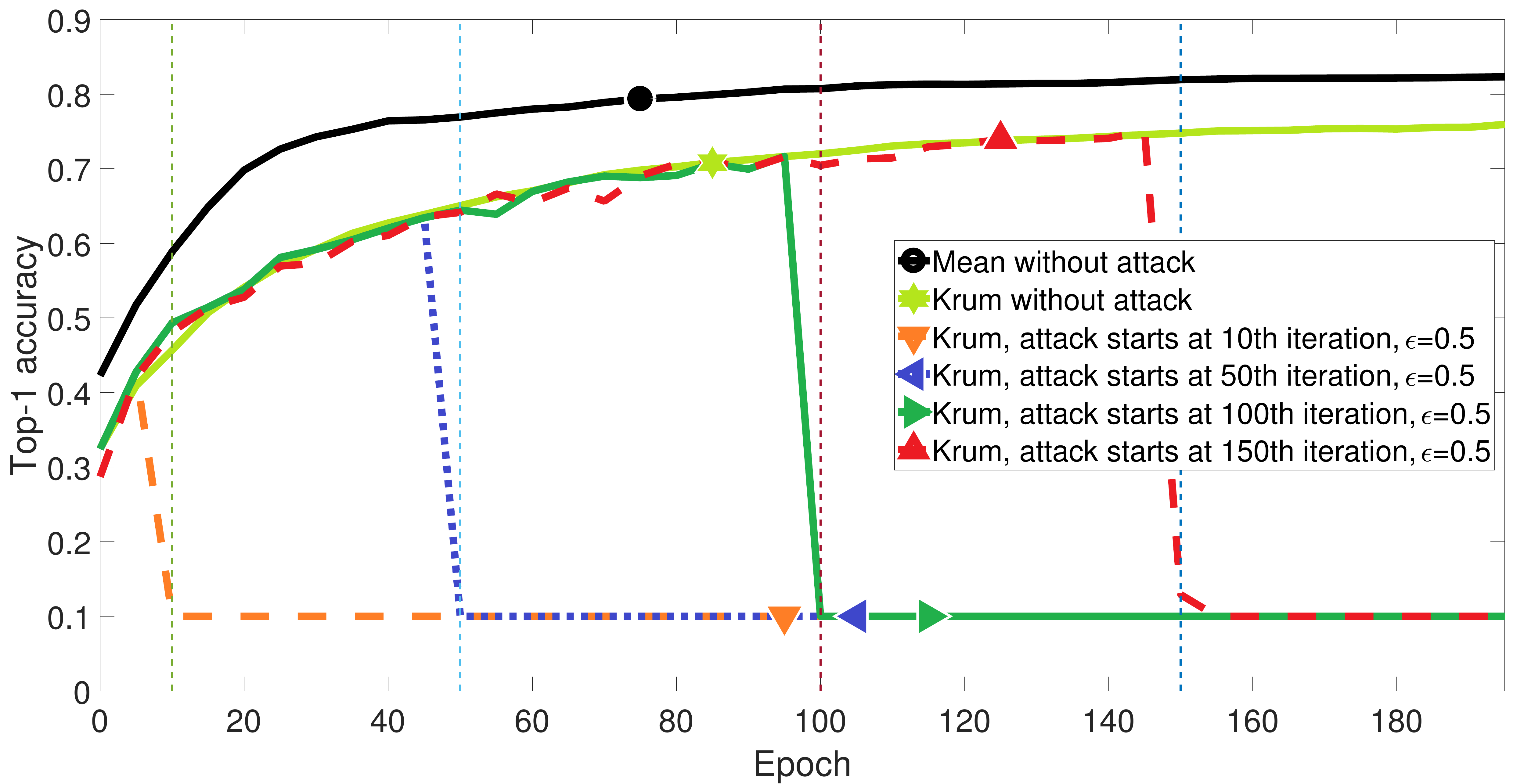}}
\subfigure[Cross Entropy on Training Set, $\epsilon=0.5$]{\includegraphics[width=0.49\textwidth,height=4.4cm]{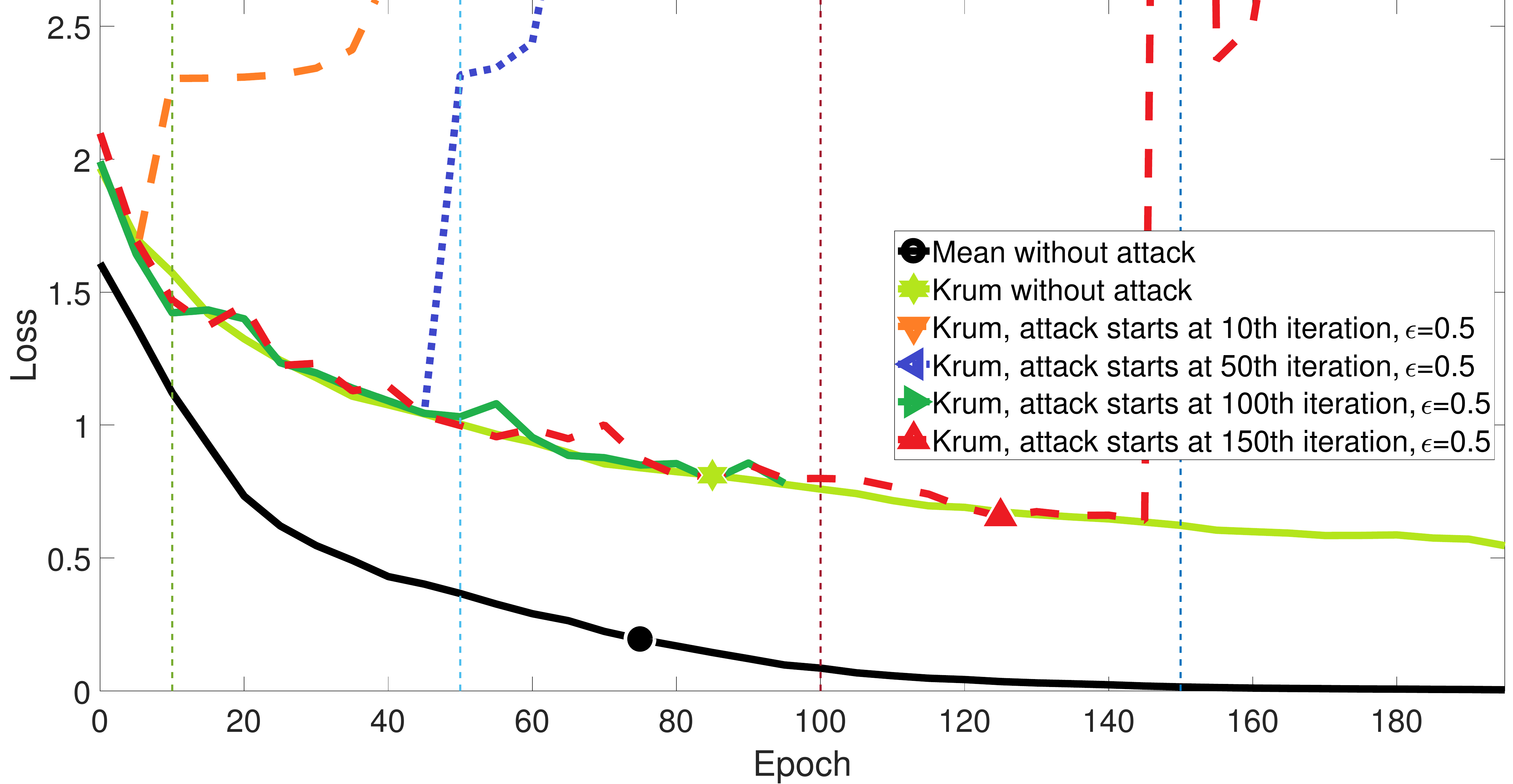}}
\subfigure[Top-1 Accuracy on Testing Set, $\epsilon=1$]{\includegraphics[width=0.49\textwidth,height=4.4cm]{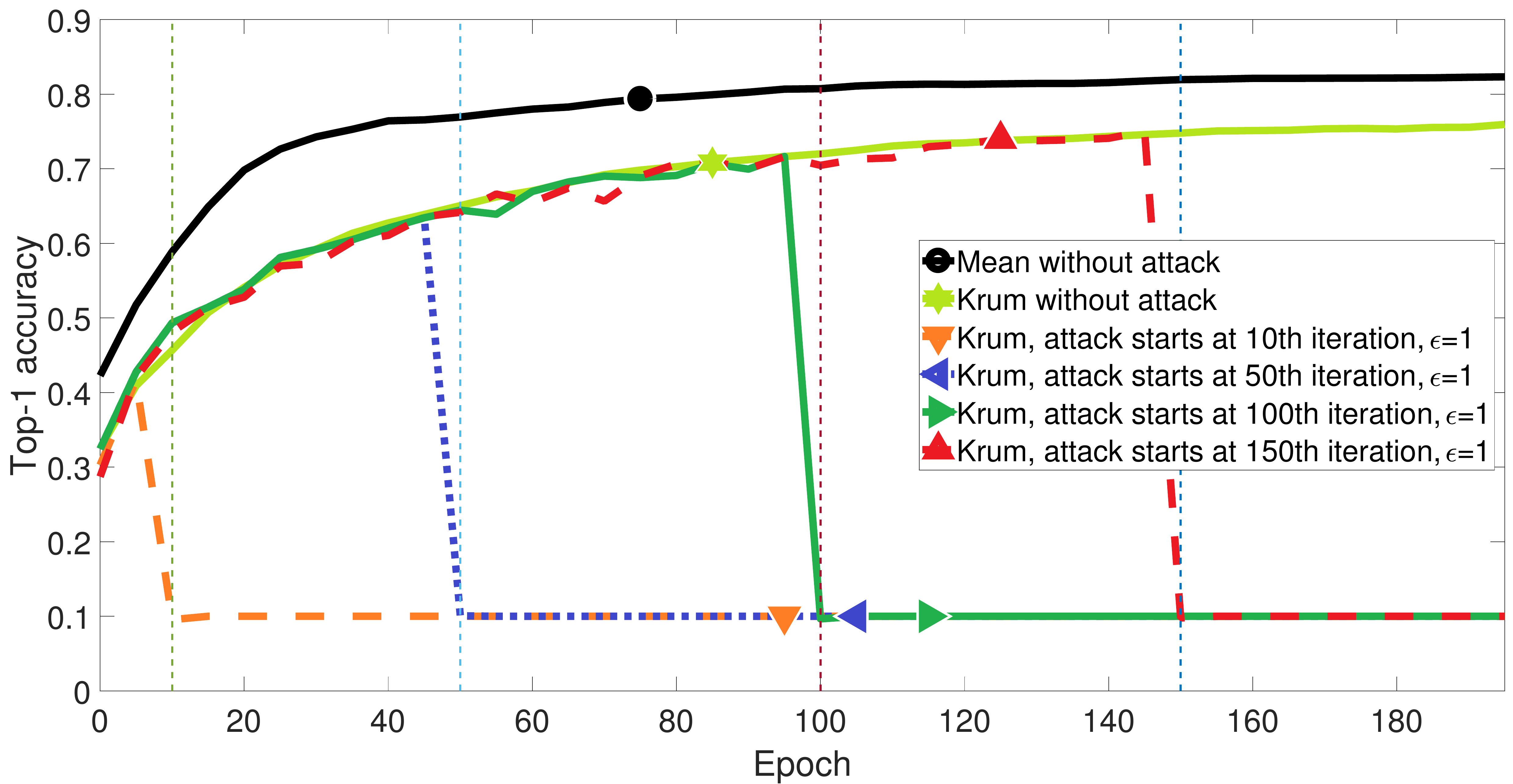}}
\subfigure[Cross Entropy on Training Set, $\epsilon=1$]{\includegraphics[width=0.49\textwidth,height=4.4cm]{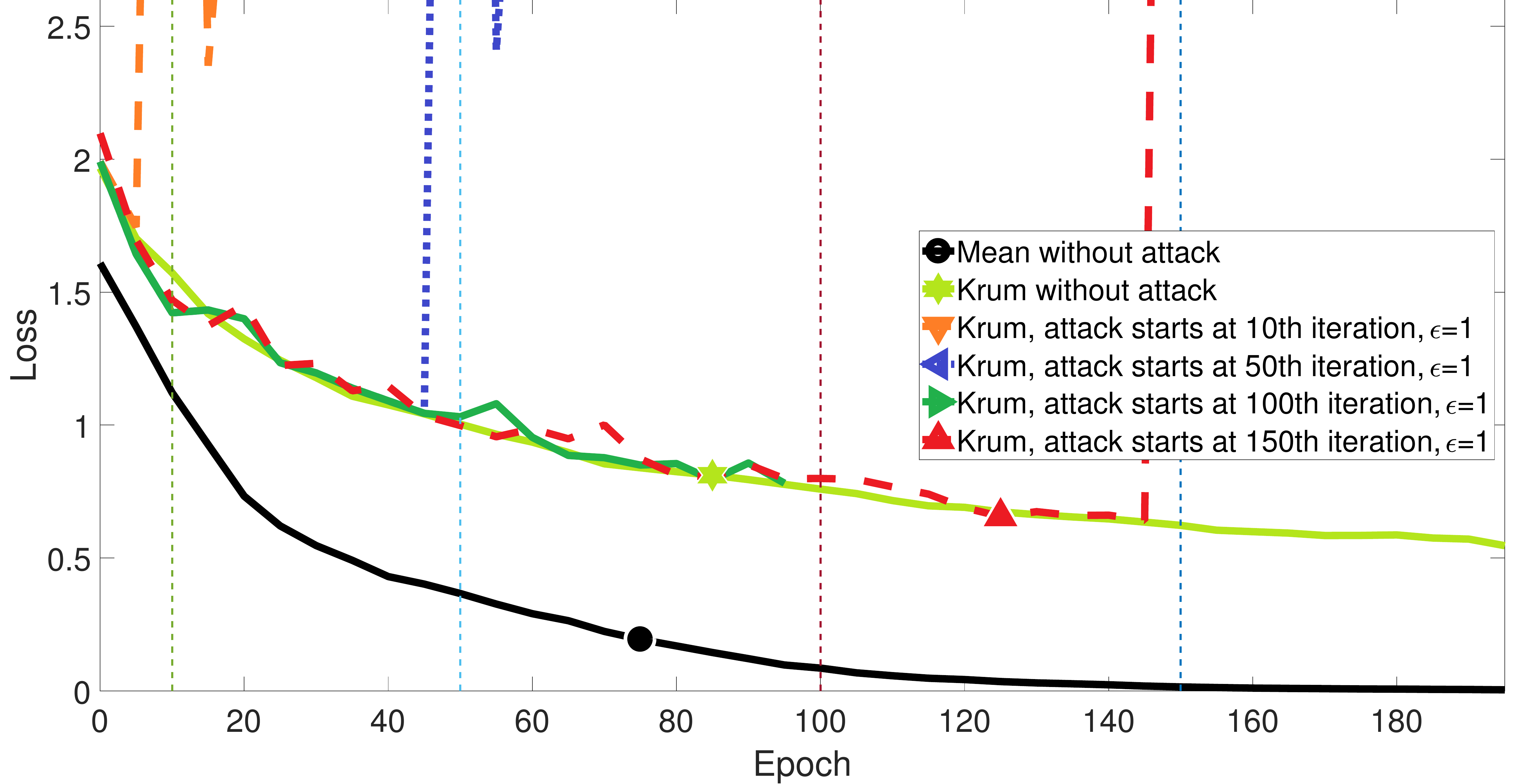}}
\subfigure[Top-1 Accuracy on Testing Set, $\epsilon=10$]{\includegraphics[width=0.49\textwidth,height=4.4cm]{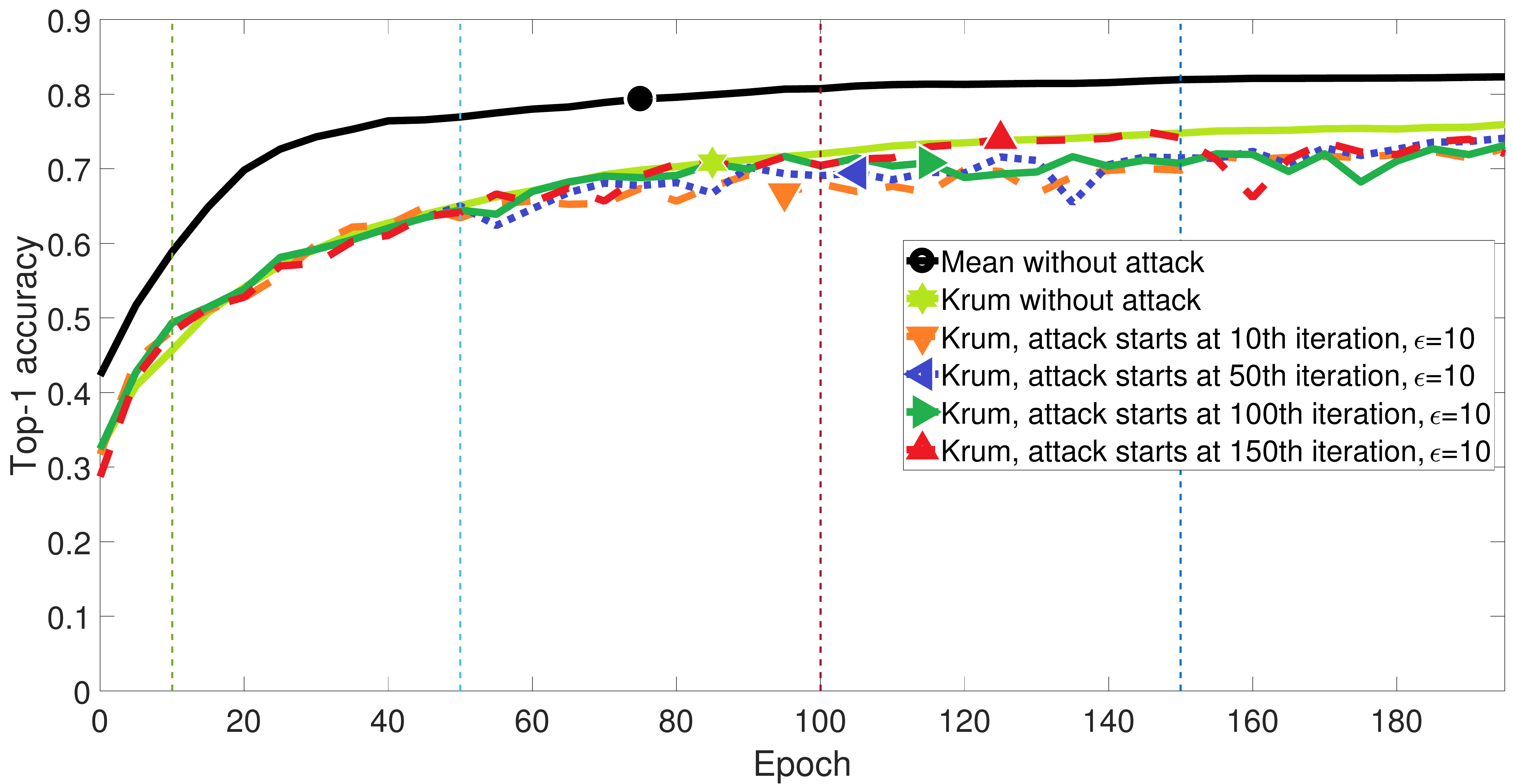}}
\subfigure[Cross Entropy on Training Set, $\epsilon=10$]{\includegraphics[width=0.49\textwidth,height=4.4cm]{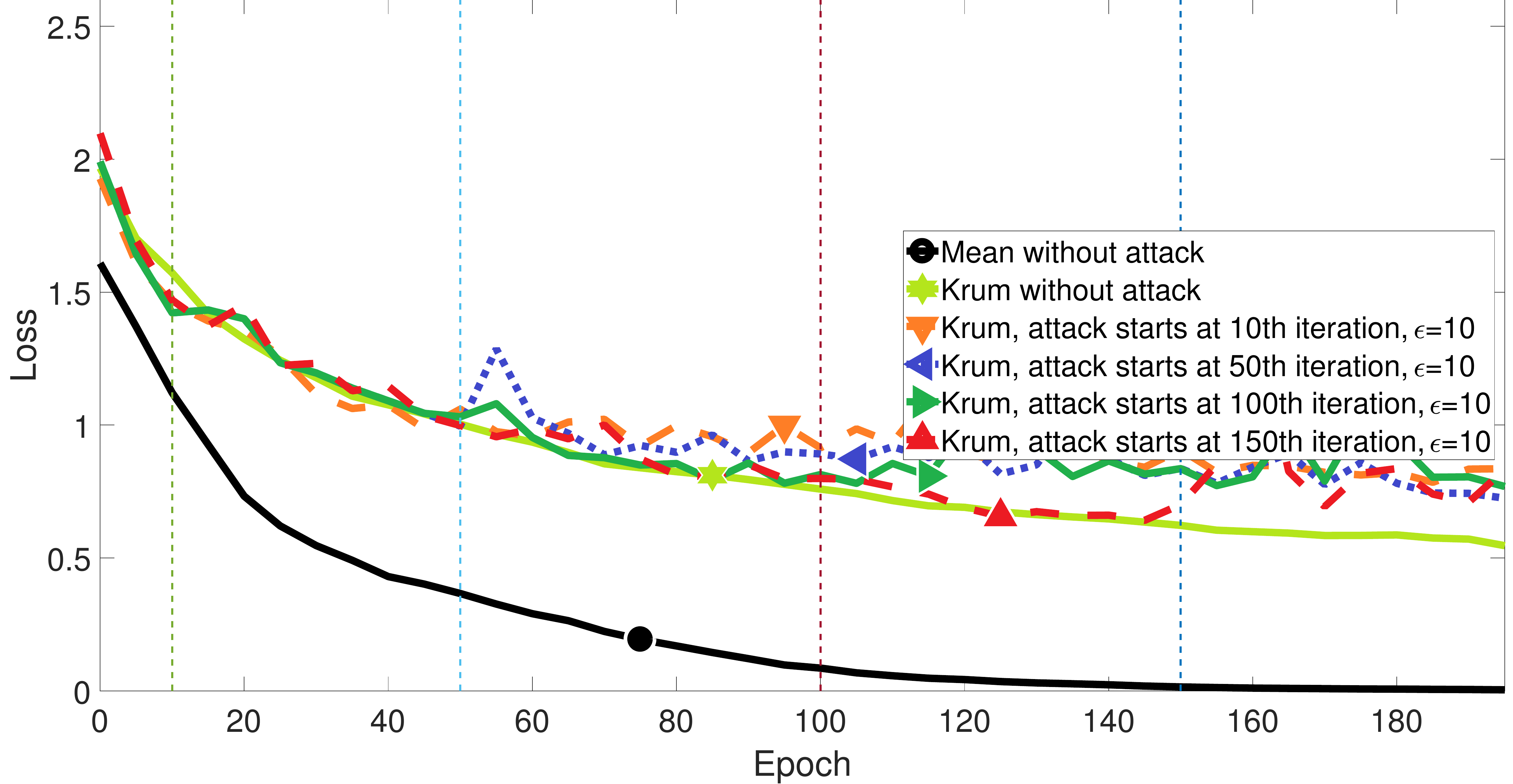}}
\caption{Convergence on training set, using \texttt{Krum} as aggregation rule. $\epsilon \in \{0.1, 0.5, 1, 10\}$.}
\label{fig:krum}
\end{figure*}

\section{CASE STUDY}

In this section, we implement special attack strategies for \texttt{Median} and \texttt{Krum}, and evaluate our attack strategies on a real-world application. The attack strategies are designed by using the intuitions underlying Theorem~\ref{thm:median} and Theorem~\ref{thm:krum}, which are mentioned in Remark~\ref{rmk:median} and Remark~\ref{rmk:krum}.

\subsection{DATASETS AND EVALUATION METRICS}
We conduct experiments on the benchmark CIFAR-10 image classification dataset~\citep{krizhevsky2009learning}, which is composed of 50k images for training and 10k images for testing. We use a convolutional neural network~(CNN) with 4 convolutional layers followed by 1 fully connected layer. The detailed network architecture can be found in the appendix. For any worker, the minibatch size for SGD is $50$.

In each experiment, we launch 25 worker processes. We repeat each experiment 10 times and take the average. We use top-1 accuracy on the testing set and the cross-entropy loss function on the training set as the evaluation metrics. 

We use the averaging, \texttt{Median}, and \texttt{Krum} without attacks as the gold standard, which are referred to as \texttt{Mean without attack}, \texttt{Median without attack}, and \texttt{Krum without attack}. We start the attack at different epochs, so that SGD can warm up and make some progress first. We include some additional experiments in the appendix.

\subsection{MEDIAN}

In each iteration, the server receives $m=25$ gradients. A randomly selected subset of $q=12$ correct gradients are replaced by Byzantine gradients. We define the set of Byzantine gradients as $\mathcal{U} = \{u_1, \ldots, u_{12}\}$, and the set of the remaining correct gradients as $\mathcal{V} = \{v_1, \ldots, v_{13}\}$. Our attack strategy is as follows:
\begin{align*}
u_1 = u_2 = \cdots = u_{12} = - \frac{\epsilon}{13} \sum_{i=1}^{13} v_i.
\end{align*}
According to Theorem~\ref{thm:median} and Remark~\ref{rmk:median}, \texttt{Median} is vulnerable to positive $\epsilon$ with large magnitude $|\epsilon|$.

We test the above attack strategy with different $\epsilon$. The results are shown in Figure~\ref{fig:median}. \texttt{Median} fails when $\epsilon > 0$. When $\epsilon = 0$, \texttt{Median} gets stuck and stops making progress. When $\epsilon < 0$, \texttt{Median} successfully defends against the attack.

\subsection{KRUM}

In each iteration, the server receives $m=25$ gradients. A randomly selected subset of $q=11$ correct gradients are replaced by Byzantine gradients. We define the set of Byzantine gradients as $\mathcal{U} = \{u_1, \ldots, u_{11}\}$, and the set of the remaining correct gradients as $\mathcal{V} = \{v_1, \ldots, v_{14}\}$. Our attack strategy is as follows:
\begin{align*}
u_1 = u_2 = \cdots = u_{11} = - \frac{\epsilon}{14} \sum_{i=1}^{14} v_i.
\end{align*}
According to Theorem~\ref{thm:krum} and Remark~\ref{rmk:krum}, \texttt{Krum} is vulnerable to positive $\epsilon$ with small magnitude $|\epsilon|$.

We test the above attack strategy with different $\epsilon$. The results are shown in Figure~\ref{fig:krum}. \texttt{Krum} fails when $\epsilon > 0$ is small. When $\epsilon$ is large enough, \texttt{Krum} successfully defends against the attack.

\subsection{DISCUSSION}

Surprisingly, both \texttt{Median} and \texttt{Krum} are more vulnerable than we expected. Note that our theorems only analyze the worst cases. There are other cases where \texttt{Median} and \texttt{Krum} can fail.

For \texttt{Median}, even if we take $\epsilon = 0$, SGD still performs badly. Theoretically, even if we do not use positive $\epsilon$, small $\epsilon$ can still enlarge the variance of SGD, which can be potentially harmful to the convergence. We can see that with large negative $\epsilon$, the defense of \texttt{Median} is successful. In our experiment, we reveal certain new vulnerabilities of \texttt{Median} in distributed synchronous SGD. The experiments conducted by \citet{yin2018byzantine} do not fail because the attacker only changes the labels of the poisoned training data by flipping a $label \in \{0, \ldots, 9\}$ to $9-label$. It is very likely that such an attack produces Byzantine gradients surrounding the correct gradients coordinate-wisely on both sides. However, according to Theorem~\ref{thm:median} and Remark~\ref{rmk:median}, an effective attack should place the Byzantine gradient on one and only one side of the correct gradients, which is the side opposite to the mean of the correct gradients, coordinate-wisely.

For \texttt{Krum}, small positive $\epsilon$ makes SGD vulnerable. Furthermore, even if we take $\epsilon = 1$, \texttt{Krum} still fails. In our experiment, we reveal certain new vulnerability of \texttt{Krum} in distributed synchronous SGD. The experiments conducted by \citet{blanchard2017machine} do not fail even though a similar attack strategy called ``omniscient'' is conducted. The reason is that, in the paper of \citet{blanchard2017machine}, the attacker ``proposes the opposite vector, scaled to a large length'', which is similar to our attack strategy with a large $\epsilon$.

Guided by our theoretical analysis, we designed efficient attack strategies for both \texttt{Median} and \texttt{Krum}. Our results show that the definition of Byzantine tolerance for distributed synchronous SGD should be revised. Using our definition of DSSGD-Byzantine tolerance, research can be conducted to design better defense techniques. 

\section{CONCLUSION}

We propose a revised definition of Byzantine tolerance for distributed synchronous SGD. With the new definition, we theoretically and empirically examine the Byzantine tolerance of two prevailing robust aggregation rules. Guided by our theoretical analysis, attack techniques can be designed to fail the aggregation rules. In the future, we hope new defense techniques can be designed using our revised definition of Byzantine tolerance. 

\newpage
\bibliography{byz}
\bibliographystyle{icml2019}

\end{document}